\documentclass{article}
\usepackage[preprint, nonatbib]{neurips_2026}
\usepackage[utf8]{inputenc}
\usepackage[T1]{fontenc}
\usepackage{mathptmx}
\usepackage[hidelinks]{hyperref}
\usepackage{url}
\usepackage{booktabs}
\usepackage{amsfonts}
\usepackage{amsmath}
\usepackage{amssymb}
\usepackage{amsthm}
\usepackage{nicefrac}
\usepackage{microtype}
\usepackage{xcolor}
\usepackage{graphicx}
\usepackage{mathtools}
\usepackage{algorithm}
\usepackage{algorithmic}
\usepackage{acronym}
\usepackage{multirow}
\usepackage{makecell}
\usepackage{subcaption}
\usepackage{wrapfig}
\usepackage{todonotes}
\usepackage{tikz}
\usetikzlibrary{calc,shapes.geometric,arrows.meta}
\usepackage[european,siunitx]{circuitikz}
\usepackage{pgfplots}
\pgfplotsset{compat=1.17}
\usepgfplotslibrary{groupplots,fillbetween}
\usepackage{csquotes}
\usepackage{caption}
\usepackage{mathdef}

\usepackage[style=numeric,natbib=true,backend=biber]{biblatex}
\newcommand{\Exp}{\mathbb{E}}

\acrodef{MLP}{Multi-Layer Perceptron}
\acrodef{TGL}{Threshold-Gated Layer}

\definecolor{DissBlue}{RGB}{31,119,180}
\definecolor{DissOrange}{RGB}{255,127,14}
\definecolor{DissGreen}{RGB}{44,160,44}
\definecolor{DissRed}{RGB}{214,39,40}
\definecolor{DissPurple}{RGB}{148,103,189}
\definecolor{DissTeal}{RGB}{23,190,207}
\definecolor{DissGridGray}{RGB}{204,204,204}

\newtheorem{theorem}{Theorem}
\newtheorem{proposition}[theorem]{Proposition}
\newtheorem{corollary}[theorem]{Corollary}
\newtheorem{definition}[theorem]{Definition}
\newtheorem{remark}[theorem]{Remark}

\addMathDef{FuncReLU}{\mathrm{ReLU}}{Rectified Linear Unit activation}{function}
\addMathDef{FuncLeakyReLU}{\mathrm{LeakyReLU}}{Leaky ReLU activation}{function}
\addMathDef{FuncSiLU}{\mathrm{SiLU}}{Sigmoid Linear Unit activation}{function}
\addMathDef{FuncGELU}{\mathrm{GELU}}{Gaussian Error Linear Unit activation}{function}
\addMathDef{FuncELU}{\mathrm{ELU}}{Exponential Linear Unit activation}{function}
\addMathDef{FuncMish}{\mathrm{Mish}}{Mish activation}{function}
\addMathDef{FuncSigmoid}{\mathrm{Sigmoid}}{Sigmoid activation}{function}
\addMathDef{FuncTanh}{\mathrm{Tanh}}{Hyperbolic tangent activation}{function}
\addMathDef{FuncHardtanh}{\mathrm{Hardtanh}}{Hard hyperbolic tangent activation}{function}
\addMathDef{FuncHardSig}{\mathrm{HardSig}}{Hard Sigmoid activation}{function}
\addMathDef{FuncTG}{\mathrm{TG}}{Threshold Gating primitive}{function}
\addMathDef{FuncTGL}{\mathrm{TGL}}{Threshold-Gated Layer (implementation)}{function}
\addMathDef{FuncReLUSix}{\mathrm{ReLU6}}{ReLU clamped at 6}{function}
\addMathDef{FuncPReLU}{\mathrm{PReLU}}{Parametric ReLU activation}{function}

\addMathDef{WeightMat}{\mathbf{W}}{Weight matrix of a layer}{matrix}
\addMathDef{InputVec}{\mathbf{x}}{Input vector to a layer}{vector}

\addMathDef{ActFunc}{\sigma}{Activation function (generic)}{function}
\addMathDef{Sharpness}{\tau}{Gate sharpness parameter}{scalar}
\addMathDef{GateThresh}{\theta}{Gate threshold}{scalar}
\addMathDef{BranchSlope}{s}{Branch slope parameter}{scalar}
\addMathDef{BranchOffset}{c}{Branch offset parameter}{scalar}
\addMathDef{NumBranches}{K}{Number of gating branches}{scalar}
\addMathDef{GateFunc}{g}{Gate function}{function}
\addMathDef{OutputAct}{a}{Layer output activation element}{scalar}
\addMathDef{BiasElem}{b}{Bias element}{scalar}
\addMathDef{WeightElem}{W}{Weight matrix element (subscripted $W_{ji}$)}{scalar}
\addMathDef{BranchWeightElem}{w}{Absorbed branch weight element $w^{(k)}_{ji} = s_k W_{ji}$}{scalar}
\addMathDef{ApproxError}{\varepsilon}{Approximation error bound}{scalar}
\addMathDef{DomainBound}{M}{Domain bound for universality}{scalar}
\addMathDef{LeakySlope}{\alpha}{Negative-slope parameter (LeakyReLU / PReLU)}{scalar}

\addMathDef{NetworkN}{\mathcal{N}}{Neural network}{set}

\addMathDef{Reals}{\mathbb{R}}{Real numbers}{set}

\addMathDefDerived{NetworkNPrime}{\mathcal{N}'}{Converted TG network}{set}{NetworkN}
\addMathDefDerived{GateBar}{\bar{g}}{Complementary gate ($1 - g$)}{function}{GateFunc}
\newtheorem{lemma}[theorem]{Lemma}

\title{Rethinking Neural Nonlinearity as Gating}
\author{%
  Muhammad Sabih, \quad Frank Hannig, \quad Jürgen Teich \\
  Department of Computer Science, Friedrich-Alexander-Universität Erlangen-Nürnberg (FAU), Germany \\
  \texttt{\{muhammad.sabih, frank.hannig, juergen.teich\}@fau.de} \\
}
\begin{document}
\maketitle
\begin{abstract}
Activation functions are considered an essential primitive for neural nonlinearity, i.e., they enable neural networks to serve as universal approximators. In this paper, we show that this nonlinearity can also be achieved by \emph{input-conditioned threshold gating through branches} as a universal primitive. We demonstrate that standard
activations---whether piecewise-linear ($\FuncReLU$, $\FuncPReLU$, $\FuncHardtanh$) or smooth
($\FuncSiLU$, $\FuncSigmoid$, $\FuncTanh$, $\FuncGELU$)---are in fact instances of a single Threshold Gating (TG) primitive. For softmax, we show that it admits an exact TG conversion via its equivalent per-element $\FuncSigmoid$ form. We then validate these equivalences by
converting pretrained networks across CNNs, transformer-based models, and
recurrent architectures, preserving model performance without requiring retraining.
Threshold Gating also enables training from scratch that goes beyond replacing existing activations, enabling gains in model compression, performance, and shorter training. We also propose a `Minimal Branch Theorem' which relates the minimum number of required branches in our primitive to the trainability of general deep neural networks. In terms of hardware implementation, TG maps to a unified implementation in the case of analog in-memory systems, addressing the bottleneck of analog-to-digital and digital-to-analog converters (ADC/DAC) that is known to significantly impact power consumption and on-chip area.
\end{abstract}

\section{Introduction}
\label{sec:intro}

Every neural network layer pairs a linear transformation with a nonlinear activation function; without it, successive layers collapse to a single linear mapping.
Activations have evolved from $\FuncSigmoid$ through $\FuncReLU$~\citep{nair2010relu} and its variants ($\FuncPReLU$~\citep{he2015prelu}, $\FuncLeakyReLU$~\citep{maas2013leakyrelu}, $\FuncELU$~\citep{clevert2016elu}) to smooth gates such as $\FuncSiLU$~\citep{ramachandran2018swish} and $\FuncGELU$~\citep{hendrycks2016gelu}, but the underlying pattern is unchanged: a nonlinear transformation following a linear operation.

As energy consumption becomes the primary concern in AI systems and transistor scaling plateaus, there is a need to find energy-efficient alternatives to the traditional digital hardware based on the Von Neumann paradigm. For example, utilizing analog in-memory computing \citep{Leroux2025} reduced attention latency and energy consumption by up to two and four orders of magnitude, respectively, compared with GPUs, for GPT-2 LLM. In analog in-memory computing specifically, ReRAM and similar nonvolatile memories encode weights as device conductances and Kirchhoff's law performs the MAC in physics~\citep{prezioso2015, shafiee2016isaac, chi2016prime}---making the MAC nearly free and shifting the bottleneck to the \emph{uniformity of the computing substrate}, analog-digital crossings, with ADCs and DACs dominating chip area and power.
The evolution of AI, on the other hand, has taken place with the traditional hardware landscape in mind, which focuses on optimizing towards minimizing FLOPs or the number of parameters of a model.
However, the bottleneck for emerging hardware technologies is different. This motivates our rethinking of neural nonlinearity. We propose a single primitive that dissolves the activation function as a category along three dimensions: \emph{choice becomes continuous} (per-neuron parameter selection rather than a discrete architectural decision), \emph{the nonlinearity becomes implicit} (fused into pre-MAC gating rather than a post-MAC nonlinearity), and \emph{the primitive is universal} (one family covers ReLU, SiLU, GELU, $\FuncSigmoid$, Tanh, recurrent gates, and softmax). The unification carries advantages in analog hardware deployment, training, model compression, and interpretability.

The paradigm of a linear transformation followed by a nonlinear one (typically an activation function) is a \emph{convention}, not a necessity. Consider a standard layer computing
$a_j = \sigma(\sum_i W_{ji}\, x_i + b_j)$, where $\sigma$ denotes the activation function (e.g., $\FuncReLU$, $\FuncSiLU$), let $x_i$ denote the inputs, $W_{ji}$ the weight from input neuron~$i$ to output neuron~$j$, and $b_j$ be the bias term for neuron~$j$.
As a notational convention, we attribute the activation to the pre-MAC input rather than to the post-MAC sum, writing $a_j = \sum_i W_{ji} \cdot \sigma(x_i) + b_j$. In a multi-layer pipeline this is equivalent: each $x_i$ entering layer $\ell$ is already the activated output of layer $\ell{-}1$, and we only re-label which side of the MAC we associate $\sigma$ with---the network computes the same function. For $\FuncReLU$,
$\sigma(x) = \max(0, x)$, i.e., the input either passes unchanged or is zeroed.
Substituting and rearranging:
\begin{equation}
a_j = \sum_i W_{ji} \cdot \FuncReLU(x_i) + b_j
    = \sum_i \Big[\mathbf{1}[x_i > 0] \cdot
      \underbrace{W_{ji}}_{w^{(1)}_{ji}}
    \;+\; \mathbf{1}[x_i \leq 0] \cdot
      \underbrace{\;0\;}_{w^{(2)}_{ji}}
    \Big] \cdot x_i \;+\; b_j,\label{eq:relu_intro}
\end{equation}
where $\mathbf{1}[x_i > 0]$ equals $1$ when $x_i$ is positive and $0$ otherwise.
The \emph{explicit} activation has disappeared. In its place, each input element
selects one of two \emph{branch weights}: $w^{(1)}_{ji} = W_{ji}$ when $x_i > 0$, or $w^{(2)}_{ji} = 0$ when $x_i \leq 0$---an explicit nonlinearity replaced by threshold gating. This hard gate is the limiting case of a soft gate that extends naturally to smooth activations ($\FuncSiLU$, $\FuncSigmoid$, $\FuncGELU$) and yields a universal primitive.
We call this \textbf{threshold gating} ($\FuncTG$), a primitive where gates inspect each
input element and select among $K$ branches, each computing an affine function $s_k x_i + c_k$ of the input, before the summation operation
(Section~\ref{sec:tgl}, illustrated in Figure~\ref{fig:tgl_neuron}).

\paragraph{Contributions.} The unification carries through five concrete results. We propose a threshold-based gating \emph{primitive} (Section~\ref{ssec::tgl_primitive}); we show \emph{universality} of standard activations as TG instances via closed-form derivations (Section~\ref{ssec:derivations}, Propositions~\ref{prop:absorption}--\ref{prop:hardtanh}, Table~\ref{tab:taxonomy}; empirically validated in Table~\ref{tab:conversion}); we prove a \emph{minimal branch theorem} (Section~\ref{ssec:minimal_branch}, Theorem~\ref{thm:minimal_branch}); we extend the primitive to \emph{softmax attention}, with lossless conversion on pretrained LLMs (Section~\ref{sec:attention}); and we demonstrate \emph{increased expressiveness} in width/depth reduction, shorter training, and interpretability.

\section{Threshold Gating as a Universal Primitive}
\label{sec:tgl}

\subsection{The Threshold Gating (TG) Primitive}
\label{ssec::tgl_primitive}

In Eq.~(\ref{eq:relu_intro}), we showed how $\FuncReLU$ can be reformulated using input-conditioned hard-gating with two branches. The key observation is that hard-gating ($\mathbf{1}[x_i > \theta]$) is the limiting
special case of soft-gating\footnote{The limit holds almost everywhere (a.e.):
$\lim_{\tau\to\infty}\sigma(\tau(x{-}\theta)) = \mathbf{1}[x>\theta]$
for $x \neq \theta$ (Appendix~\ref{app:measure_zero}).} as shown in Figure~\ref{fig:sharpness}.
We define the \emph{gate function} $g(x_i;\,\tau,\,\theta)$, where $x_i$ is the input and $\tau,\,\theta$ are parameters:
\begin{equation}
\label{eq:softgate}
g(x_i;\,\tau,\,\theta) \;=\; \sigma\!\big(\tau \cdot (x_i - \theta)\big)
\;=\; \frac{1}{1 + e^{-\tau \cdot (x_i - \theta)}},
\end{equation}
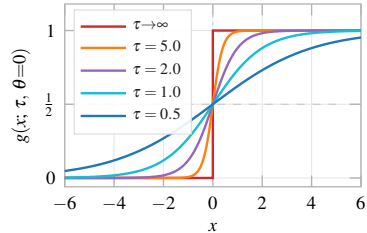
\begin{wrapfigure}{r}{0.45\textwidth}
\centering
\vspace{-0.1cm}
\begin{tikzpicture}
\begin{axis}[
    width=5.5cm,
    height=4.0cm,
    xmin=-6, xmax=6,
    xtick={-6,-4,-2,0,2,4,6},
    clip=false,
    title style={font=\scriptsize},
    ymin=-0.06, ymax=1.18,
    xlabel={$x$},
    ylabel={$g(x;\,\tau,\,\theta{=}0)$},
    xlabel style={font=\scriptsize, inner sep=1pt},
    ylabel style={font=\scriptsize, inner sep=1pt},
    tick label style={font=\scriptsize},
    ytick={0, 0.5, 1},
    yticklabels={$0$, $\frac{1}{2}$, $1$},
    grid=major,
    grid style={gray!20, line width=0.3pt},
    axis line style={gray!60, line width=0.5pt},
    tick style={gray!60, line width=0.5pt},
    legend style={
        at={(0.03, 0.97)},
        anchor=north west,
        font=\tiny,
        fill=white,
        fill opacity=0.9,
        draw=gray!30,
        inner sep=2pt,
        row sep=-1pt,
        legend columns=1,
    },
    legend cell align=left,
    every axis plot/.append style={line width=0.9pt},
]
\addplot[dashed, gray!45, line width=0.4pt, domain=-5:5, forget plot] {0.5};
\addplot[dotted, gray!45, line width=0.4pt, forget plot]
    coordinates {(0,-0.06)(0,1.18)};
\addplot[DissRed] coordinates {(-6,0)(-0.001,0)(-0.001,1)(6,1)};
\addlegendentry{$\tau{\to}\infty$}
\addplot[DissOrange, smooth, domain=-6:6, samples=200] {1/(1+exp(-5*x))};
\addlegendentry{$\tau=5.0$}
\addplot[DissPurple, smooth, domain=-6:6, samples=200] {1/(1+exp(-2*x))};
\addlegendentry{$\tau=2.0$}
\addplot[DissTeal, smooth, domain=-6:6, samples=200] {1/(1+exp(-1*x))};
\addlegendentry{$\tau=1.0$}
\addplot[DissBlue, smooth, domain=-6:6, samples=200] {1/(1+exp(-0.5*x))};
\addlegendentry{$\tau=0.5$}
\end{axis}
\end{tikzpicture}
\vspace{-0.2cm}
\caption{Gate sharpness ($\tau$) continuum.}
\label{fig:sharpness}
\end{wrapfigure}
where $\theta$ is the threshold and $\tau > 0$ controls the softness: large $\tau$ approaches a hard step, small $\tau$ blends the branches smoothly, enabling exact replacement of smooth activations like $\FuncSiLU$ that the hard indicator cannot.
Generalizing to $K$ regions with $K{-}1$ thresholds $\theta_1 < \cdots < \theta_{K-1}$, each region~$k$ carries a soft gate $g_k(x_i) \in [0,1]$ and an affine branch $s_k \cdot x_i + c_k$. The gates form a partition of unity, $\sum_k g_k = 1$; unlike a hard one-hot selection, multiple gates may be simultaneously active with fractional values, blending adjacent branches smoothly near each threshold.

The $\FuncTG$ primitive is:
\begin{equation}
\label{eq:tg_primitive}
\boxed{
y_i = \sum_{k=1}^{K} g_k(x_i)\,(s_k \cdot x_i + c_k).
}
\end{equation}
This replaces $\ActFunc(x_i)$ with parameters $(\Sharpness, \GateThresh, \{s_k\}, \{c_k\})$: $s_k$ scales input magnitude through branch~$k$ (input-dependent), $c_k$ contributes a constant (input-independent). Substituting into $a_j = \sum_i W_{ji}\,y_i + b_j$ gives:

\begin{equation}
\label{eq:fused}
\boxed{
a_j = \left[\sum_{i}\sum_{k=1}^{K}
  g_k(x_i)\big(w^{(k)}_{ji}\,x_i + c^{(k)}_{ji}\big)\right] + b_j.
}
\end{equation}
In general, the branch weights $w^{(k)}_{ji}$ are independent learnable parameters---each branch keeps its own weight per synapse (Figure~\ref{fig:tgl_neuron}). The important special case shares a common weight across branches of a synapse, $w^{(k)}_{ji} = s_k \cdot W_{ji}$:
\begin{equation}
w^{(k)}_{ji} \triangleq s_k \cdot W_{ji}, \quad
c^{(k)}_{ji} \triangleq c_k \cdot W_{ji}.
\label{eq:absorbed_weights}
\end{equation}
This enables a lossless
conversion from pretrained networks (Section~\ref{ssec:conversion}):
the learned weights $W_{ji}$ are preserved unchanged, and only the
activation is re-parameterized through $(s_k, c_k)$.
Since Eq.~(\ref{eq:tg_primitive}) is element-wise, the same primitive
applies unchanged to convolutions, recurrent cells, or any layer whose
activations act per element.

\begin{wraptable}[10]{r}{0.55\textwidth}
\vspace{-0.8em}
\caption{Progressive TG tiers of the general primitive. \textbf{Bold} marks configurations exercised in Section~\ref{sec:learning}.}
\label{tab:design-space}
\centering
\scriptsize
\setlength{\tabcolsep}{4pt}
\renewcommand{\arraystretch}{1.25}
\begin{tabular}{@{}lcl@{}}
\toprule
\textbf{Tier} & \textbf{Overhead}$^\dagger$ & \textbf{$(\tau,\theta)$ sharing} \\
\midrule
General (untied) & $3K{\times}$ & per-synapse \\
$\FuncTG_S$ (untied $w$) & $2K{\times}$
  & \textbf{synapse} / channel-pair / channel \\
$\FuncTG_W$ ($K$-branch wts) & $K{\times}$
  & \textbf{layer} / \textbf{chan} / neuron \\
$\FuncTG_A$ (act-only) & $\leq 1\%$
  & net / layer / chan / \textbf{neuron} / syn \\
\bottomrule
\end{tabular}
\par\vspace{1mm}\noindent\tiny

$^\dagger$Relative to baseline $N = C_\text{in} C_\text{out}$. General: $K$ untied $w^{(k)}{ji}$ + $K$ untied $c^{(k)}{ji}$ + $K$ gate scalars per synapse (one shared $\tau$, $K{-}1$ thresholds). Successive tiers absorb $c$ ($\FuncTG_S$), coarsen the gate ($\FuncTG_W$), and absorb $w$ ($\FuncTG_A$); see Eq.~(\ref{eq:absorbed_weights}).
\vspace{-0.8em}
\end{wraptable}

Eq.~(\ref{eq:fused}) admits independent parameters per-synapse: branch weights $w^{(k)}_{ji}$, branch constants $c^{(k)}_{ji}$, and gate parameters ($K{-}1$ thresholds $\theta^{(k)}_{ji}$ and one shared sharpness $\tau_{ji}$). We refer to this as the \emph{general} primitive; its parameter cost is $3K \cdot N$ per layer (Table~\ref{tab:design-space}). But practically, we constrain the fully general primitive as follows. \emph{$\FuncTG_S$:} absorb branch constants via $c^{(k)}_{ji} = c_k W_{ji}$, leaving $K$ untied branch weight matrices and per-synapse gate parameters. \emph{$\FuncTG_W$:} additionally coarsen the gate parameters from per-synapse to per-channel or per-layer sharing, but the $K$ branches still carry independent weights. \emph{$\FuncTG_A$:} additionally tie the branch weights via $w^{(k)}_{ji} = s_k W_{ji}$ (Eq.~\ref{eq:absorbed_weights}); the $K$ branches now share a single $W$ and the primitive collapses to a scalar-to-scalar function $\FuncTG(x_i) = \sum_k g_k(x_i)(s_k x_i + c_k)$ applied per element before the unmodified MAC---a drop-in replacement for $\FuncReLU$/$\FuncGELU$/$\FuncSiLU$. The basic form is illustrated in Figure~\ref{fig:tgl_neuron}; the three tiers and their parameter costs are summarized in Table~\ref{tab:design-space}.

\begin{wrapfigure}[13]{r}{0.70\textwidth}
\vspace{-0.6em}
\centering

\newcommand{\plotvspace}{2.5}
\newcommand{\plothspace}{1.8}
\newcommand{\plotxoffset}{10.8}
\newcommand{\plotyoffset}{-2}
\newcommand{\plotwidth}{3cm}
\newcommand{\plotheight}{2.6cm}

\begin{tikzpicture}[
  inputnode/.style={circle, draw=DissGreen!80, fill=DissGreen!8,
                    line width=0.8pt, minimum size=6mm, inner sep=0pt},
  sumnode/.style={circle, draw=black!70, fill=black!4,
                  line width=0.8pt, minimum size=6mm, inner sep=0pt},
  outputnode/.style={circle, draw=DissOrange!80, fill=DissOrange!10,
                     line width=0.8pt, minimum size=6mm, inner sep=0pt},
  sigmabox/.style={draw=DissPurple!80, fill=DissPurple!10, line width=0.8pt,
                   rounded corners=2pt, minimum width=8mm,
                   minimum height=8mm, font=\scriptsize},
  smallgate/.style={draw=DissPurple!80, fill=DissPurple!8, line width=0.7pt,
                    rounded corners=1.5pt, minimum width=4mm,
                    minimum height=4mm, font=\tiny, inner sep=1pt},
  threshbox/.style={draw=DissRed!80, fill=DissRed!8, line width=0.7pt,
                    diamond, shape aspect=1.5, minimum size=6mm,
                    inner sep=0pt, font=\tiny},
  branchbox/.style={draw=DissBlue!80, fill=DissBlue!6, line width=0.7pt,
                    rounded corners=1.5pt, font=\tiny, inner sep=2pt},
  blendnode/.style={draw=DissGreen!80, fill=DissGreen!8, line width=0.7pt,
                    circle, minimum size=5mm, font=\tiny, inner sep=0pt},
  >=stealth,
]


\node[font=\scriptsize\bfseries, text=black!70] at (2.1, 2.0) {Conventional};

\node[inputnode, font=\scriptsize] (ci1) at (0, 1.0)   {$x_1$};
\node[inputnode, font=\scriptsize] (ci2) at (0, 0)     {$x_i$};
\node[font=\normalsize, text=black!30] at (0, -0.7) {$\vdots$};
\node[inputnode, font=\scriptsize] (cin) at (0, -1.5)  {$x_n$};
\node[sumnode, font=\scriptsize] (csum) at (1.4, 0) {$\Sigma{+}b$};
\node[sigmabox] (csig) at (2.8, 0) {$\sigma$};
\node[outputnode, font=\scriptsize] (cout) at (4.2, 0) {$a_j$};

\draw[->, black!70, thick] (ci1) -| (csum);
\draw[->, black!70, thick] (ci2) -- (csum);
\draw[->, black!70, thick] (cin) -| (csum);
\draw[->, black!70, thick] (csum) -- node[above, font=\tiny] {$z_j$} (csig);
\draw[->, black!70, thick] (csig) -- (cout);

\node[font=\tiny, above=2pt] at ($(ci1)!0.5!(csum)$) {$W_{j1}$};
\node[font=\tiny, above] at ($(ci2)!0.5!(csum)$) {$W_{ji}$};
\node[font=\tiny, above] at (0.7, -1.5) {$W_{jn}$};

\node[font=\scriptsize, anchor=west, text=black!70] at (-0.3, -2.3)
  {$a_j = \sigma\!\big(\sum_{i} W_{ji}\, x_i + b_j\big)$};

\draw[dashed, black!30, line width=0.8pt] (4.55, 2.2) -- (4.55, -2.8);


\pgfmathsetmacro{\rx}{5.1}

\node[font=\scriptsize\bfseries, text=black!80] at ({\rx + 2.3}, 2.0) {$\FuncTG$};

\node[inputnode, font=\scriptsize] (ti1) at ({\rx - 0.15}, 1.5)  {$x_1$};
\node[inputnode, font=\scriptsize] (ti2) at ({\rx - 0.15}, 0)    {$x_i$};
\node[font=\normalsize, text=black!30] at ({\rx - 0.15}, 0.75) {$\vdots$};
\node[font=\normalsize, text=black!30] at ({\rx - 0.15}, -0.75) {$\vdots$};
\node[inputnode, font=\scriptsize] (tin) at ({\rx - 0.15}, -1.5) {$x_n$};

\draw[line width=0.8pt, dash pattern=on 3pt off 3pt, DissBlue!40]
  ({\rx + 0.5}, -1.05) rectangle ({\rx + 3.2}, 1.05);

\pgfmathsetmacro{\gxpos}{\rx + 0.5 + (3.2-0.5)/2}
\node[smallgate] (g1) at (\gxpos, 1.5) {$g_1$};
\draw[->, black!70, thick] (ti1) -- (g1);

\node[smallgate] (gn) at (\gxpos, -1.5) {$g_n$};
\draw[->, black!70, thick] (tin) -- (gn);

\node[threshbox] (thr) at ({\rx + 0.8}, 0) {$\gtrless\!\theta$};
\draw[->, black!70, thick] (ti2) -- (thr);

\node[branchbox, anchor=west] (br1) at ({\rx + 1.4}, 0.7)
  {$w^{(1)}_{ji}x_i{+}c^{(1)}_{ji}$};
\node[font=\small, text=black!30] at ({\rx + 2.0}, 0) {$\vdots$};
\node[branchbox, anchor=west] (brK) at ({\rx + 1.4}, -0.7)
  {$w^{(K)}_{ji}x_i{+}c^{(K)}_{ji}$};

\node[blendnode] (bld) at ({\rx + 2.8}, 0) {$+$};

\draw[->, DissRed!60, thick] (thr) |- (br1);
\draw[->, DissRed!60, thick] (thr) |- (brK);
\draw[->, DissBlue!60] (br1) -- (bld);
\draw[->, DissBlue!60] (brK) -- (bld);

\node[sumnode, font=\scriptsize] (tsum) at ({\rx + 3.8}, 0) {$\Sigma{+}b$};
\node[outputnode, font=\scriptsize] (tout) at ({\rx + 4.8}, 0) {$a_j$};

\draw[->, black!70, thick] (g1.east) -| (tsum);
\draw[->, black!70, thick] (bld.east) -- (tsum);
\draw[->, black!70, thick] (gn.east) -| (tsum);
\draw[->, black!70, thick] (tsum) -- (tout);

\node[font=\scriptsize, anchor=west, text=black!70]
  at ({\rx - 0.3}, -2.3)
  {$a_j = \sum_{i} \sum_{k} g_k(x_i) \big(w^{(k)}_{ji} x_i {+} c^{(k)}_{ji}\big) {+} b_j$};

\end{tikzpicture}
\vspace{-0.7em}
\caption{Conventional (post-summation $\sigma$) vs.\ gating $\FuncTG$ (general primitive, per-synapse branches); the primitive applies to any layer.}
\label{fig:tgl_neuron}
\end{wrapfigure}
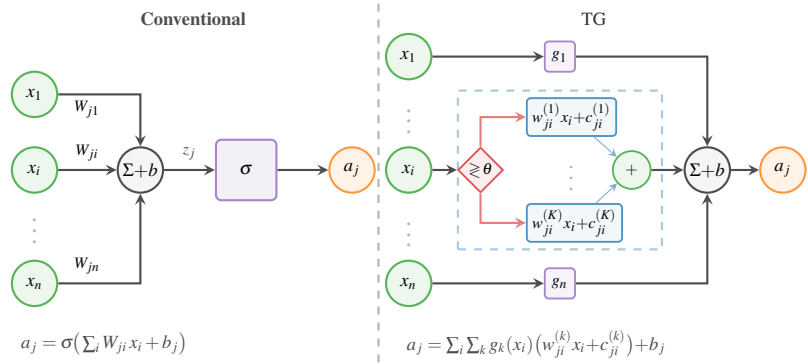

Eqs.~(\ref{eq:tg_primitive}) and (\ref{eq:absorbed_weights}) serve two roles. As a \emph{decomposition}, the TG primitive reveals the gating structure already implicit in any pretrained activation $\ActFunc$ (Section~\ref{ssec:derivations}; conversion in Section~\ref{ssec:conversion}). As a \emph{learnable activation}, the parameters $(\Sharpness, \GateThresh, s_k, c_k)$ become trainable degrees of freedom.

\subsection{Standard Activations as TG Instances}
\label{ssec:derivations}

We now show that common activations are special cases of
Eq.~(\ref{eq:tg_primitive}), also allowing for conversion of standard pretrained models to $\FuncTG$. Parameters for additional
activations are given in Table~\ref{tab:taxonomy}.

\noindent\textit{Notational convention.}\; The standard
convention writes
$\sigma(\WeightMat\InputVec{+}b)$---activation on the
output. We write
$\WeightMat\,\sigma(\InputVec){+}b$---activation on the
input, as is natural for $\FuncTG$. These are equivalent: the activated output of
layer~$\ell$ is the input to layer~$\ell{+}1$.

\begin{proposition}[Absorption]
\label{prop:absorption}
$\FuncReLU$ is a TG instance with $K{=}2$,
$\tau{\to}\infty$, $\theta{=}0$,
$s{=}(1,0)$, $c{=}(0,0)$.
\end{proposition}
\begin{proof}
$\max(0, x_i) = \mathbf{1}[x_i{>}0] \cdot x_i$.
Substituting into Eq.~(\ref{eq:tg_primitive}):
\begin{equation}
y_i = \underbrace{\mathbf{1}[x_i{>}0]}_{g_1}
(\underbrace{1}_{s_1}{\cdot}\,x_i +
\underbrace{0}_{c_1})
\;+\;
\underbrace{\mathbf{1}[x_i{\leq}0]}_{\bar{g}_1}
(\underbrace{0}_{s_2}{\cdot}\,x_i +
\underbrace{0}_{c_2})
= \FuncReLU(x_i). \qedhere
\end{equation}
\end{proof}

\begin{proposition}
\label{prop:silu}
$\FuncSiLU$ is a TG instance with $K{=}2$,
$\tau{=}1$, $\theta{=}0$,
$s{=}(1,0)$, $c{=}(0,0)$.
\end{proposition}
\begin{proof}
$\FuncSiLU(x) = x \cdot \sigma(x)$, where
$\sigma(x_i)$ is itself a gate at $\theta{=}0$,
$\tau{=}1$:
\begin{equation}
y_i = \underbrace{\sigma(x_i)}_{g_1}
(\underbrace{1}_{s_1}{\cdot}\,x_i +
\underbrace{0}_{c_1})
\;+\;
\underbrace{(1{-}\sigma(x_i))}_{\bar{g}_1}
(\underbrace{0}_{s_2}{\cdot}\,x_i +
\underbrace{0}_{c_2})
= \FuncSiLU(x_i). \qedhere
\end{equation}
\end{proof}

\noindent $\FuncSiLU$ and $\FuncReLU$ differ only in
$\tau$: soft ($\tau{=}1$) versus hard
($\tau{\to}\infty$). Both conversions are exact.

\begin{proposition}
\label{prop:tanh}
$\FuncTanh$ is a TG instance with $K{=}2$,
$\tau{=}2$, $\theta{=}0$,
$s{=}(0,0)$, $c{=}(+1,-1)$.
\end{proposition}
\begin{proof}
$\tanh(x) = 2\sigma(2x) - 1 =
\sigma(2x)(+1) + (1{-}\sigma(2x))(-1)$:
\begin{equation}
y_i = \underbrace{\sigma(2x_i)}_{g_1}
(\underbrace{0}_{s_1}{\cdot}\,x_i +
\underbrace{1}_{c_1})
\;+\;
\underbrace{(1{-}\sigma(2x_i))}_{\bar{g}_1}
(\underbrace{0}_{s_2}{\cdot}\,x_i +
\underbrace{(-1)}_{c_2})
= \tanh(x_i). \qedhere
\end{equation}
\end{proof}

\noindent This extends to $\FuncSigmoid$
($\tau{=}1$, $c_1{=}1$, $c_2{=}0$).

\begin{proposition}
\label{prop:gelu}
$\FuncGELU$ is approximately a TG instance with
$K{=}2$, $\tau{=}1.702$, $\theta{=}0$,
$s{=}(1,0)$, $c{=}(0,0)$, with ${\leq}0.3\%$
element-wise error.
\end{proposition}
\begin{proof}
$\FuncGELU(x) = x \cdot \Phi(x)$ where $\Phi$ is the
Gaussian CDF. Using
$\Phi(x) \approx \sigma(1.702x)$~\citep{hendrycks2016gelu}:
\begin{equation}
y_i = \underbrace{\sigma(1.702\,x_i)}_{g_1}
(\underbrace{1}_{s_1}{\cdot}\,x_i +
\underbrace{0}_{c_1})
\;+\;
\underbrace{(1{-}\sigma(1.702\,x_i))}_{\bar{g}_1}
(\underbrace{0}_{s_2}{\cdot}\,x_i +
\underbrace{0}_{c_2})
\approx \FuncGELU(x_i). \qedhere
\end{equation}
\end{proof}

\noindent The residual error can be mitigated through
per-layer calibration (Appendix~\ref{app:tau_matching}).

\begin{proposition}
\label{prop:hardtanh}
$\FuncHardtanh$ is a TG instance with $K{=}3$,
$\tau{\to}\infty$, $\theta{=}(-1,+1)$,
$s{=}(0,1,0)$, $c{=}(-1,0,+1)$.
\end{proposition}
\begin{proof}
$\FuncHardtanh(x) = \max(-1, \min(1, x))$
requires three regions at $\theta_1{=}{-}1$,
$\theta_2{=}{+}1$, with gates
$g_1 = \mathbf{1}[x_i{<}{-}1]$,
$g_2 = \mathbf{1}[|x_i|{\leq}1]$,
$g_3 = \mathbf{1}[x_i{>}1]$:
\begin{equation}
y_i = g_1(\underbrace{0}_{s_1}{\cdot}\,x_i + \underbrace{(-1)}_{c_1})
\;+\; g_2(\underbrace{1}_{s_2}{\cdot}\,x_i + \underbrace{0}_{c_2})
\;+\; g_3(\underbrace{0}_{s_3}{\cdot}\,x_i + \underbrace{1}_{c_3})
= \FuncHardtanh(x_i).
\end{equation}
\end{proof}
\noindent The two saturation branches enforce output
bounds---a role hardware may perform for free (e.g.,
supply rail clamping), reducing $K{=}3$ to $K{=}2$. The
same reduction applies to $\FuncHardSig$ and
$\FuncReLUSix$.

In practice, $K$ is small: as Table~\ref{tab:taxonomy}
shows, all standard activations decompose with $K{=}2$,
and only the saturating activations
($\FuncHardtanh$, $\FuncHardSig$, $\FuncReLUSix$)
nominally require $K{=}3$. The reason is structural:
their constant-output regions have $s_k{=}0$, making the
contribution $c^{(k)}_{ji} = c_k \cdot W_{ji}$
input-independent---e.g., $\tanh(x_i) \to \pm 1$ implies
$W_{ji} \cdot \tanh(x_i) \to \pm W_{ji}$, represented by
$c_k = \pm 1$. Hardware providing natural output clamping
absorbs this role for free, collapsing $K$ back to $2$.

\begin{table}[t]
\caption{TG parameters for standard activations. Each row gives
$(K, \tau, \mathbf{s}, \mathbf{c})$ for Eq.~(\ref{eq:tg_primitive}).
The $\tau$ values for $\FuncSiLU$, $\FuncSigmoid$, and $\FuncTanh$
follow from the matching condition of
Lemma~\ref{lem:local_residual}; for $\FuncGELU$ we use the
conventional Hendrycks--Gimpel value $1.702$
(Appendix~\ref{app:tau_matching}).}
\label{tab:taxonomy}
\centering
\scriptsize
\setlength{\tabcolsep}{2.5pt}
\renewcommand{\arraystretch}{1.25}
\begin{tabular}{@{}llccllllcl@{}}
\toprule
\textbf{Activation} & \textbf{Family} & $K$ & $\tau$ & $\boldsymbol{\theta}$
  & $\mathbf{s}$ & $\mathbf{c}$ & \textbf{Clamp}
  & \textbf{Exact}
  & $y_i = \sum_k g_k(s_k x_i + c_k)$ \\
\midrule
$\FuncReLU$       & Mult. & 2 & $\infty$ & $0$ & $1,\, 0$
  & $0,\, 0$ & --- & \checkmark
  & $g_\infty \cdot x_i$ \\
$\FuncLeakyReLU$  & Mult. & 2 & $\infty$ & $0$ & $1,\, \alpha$
  & $0,\, 0$ & --- & \checkmark
  & $g_\infty \cdot x_i + \bar{g}_\infty \cdot \alpha x_i$ \\
$\FuncSiLU$       & Mult. & 2 & $1.0$ & $0$ & $1,\, 0$
  & $0,\, 0$ & --- & \checkmark
  & $g_1 \cdot x_i$ \\
\midrule
$\FuncGELU$       & Mult. & 2 & $1.702$\footnotemark & $0$ & $1,\, 0$
  & $0,\, 0$ & --- & ${\approx}0.3\%$
  & $g_{1.7} \cdot x_i$ \\
$\FuncGELU$       & Mult. & 3 & cal.    & cal. & cal.
  & cal. & --- & ${\approx}0.04\%$
  & $\sum_{k=1}^{3} g_k(s_k x_i + c_k)$ \\
\midrule
$\FuncSigmoid$    & Const. & 2 & $1.0$ & $0$  & $0,\, 0$
  & $1,\, 0$ & --- & \checkmark
  & $g_1(1)$ \\
$\FuncTanh$       & Const. & 2 & $2.0$ & $0$ & $0,\, 0$
  & ${+}1,\, {-}1$ & --- & \checkmark
  & $g_2(+1) + \bar{g}_2(-1)$ \\
\midrule
$\FuncHardtanh$   & Sat. & 3 & $\infty$ & ${-}1,{+}1$ & $0,\,1,\,0$
  & $-1,\,0,\,{+}1$ & --- & \checkmark
  & $g_1(-1) + g_2 \cdot x_i + g_3(+1)$ \\
$\FuncHardtanh$   & Sat. & 2 & $\infty$ & $0$ & $1,\, 1$
  & $0,\, 0$ & $[-1,\,1]$ & \checkmark
  & $g_\infty \cdot x_i + \bar{g}_\infty \cdot x_i$\;\textrm{+clamp} \\
\midrule
$\FuncHardSig$ & Sat. & 3 & $\infty$ & $-3,\,+3$
  & $0,\,\nicefrac{1}{6},\,0$
  & $0,\,\nicefrac{1}{2},\,1$ & --- & \checkmark
  & $g_1(0) + g_2(\tfrac{1}{6} x_i {+} \tfrac{1}{2}) + g_3(1)$ \\
$\FuncHardSig$ & Sat. & 2 & $\infty$ & $0$
  & $\nicefrac{1}{6},\,\nicefrac{1}{6}$
  & $\nicefrac{1}{2},\,\nicefrac{1}{2}$ & $[0,\,1]$ & \checkmark
  & $g_\infty(\tfrac{1}{6} x_i {+} \tfrac{1}{2}) + \bar{g}_\infty(\tfrac{1}{6} x_i {+} \tfrac{1}{2})$\;\textrm{+clamp} \\
\midrule
$\FuncReLUSix$      & Sat. & 3 & $\infty$ & $0,6$ & $0,\,1,\,0$
  & $0,\,0,\,6$ & --- & \checkmark
  & $g_1(0) + g_2 \cdot x_i + g_3(6)$ \\
$\FuncReLUSix$      & Sat. & 2 & $\infty$ & $0$ & $1,\, 1$
  & $0,\, 0$ & $[0,\,6]$ & \checkmark
  & $g_\infty \cdot x_i + \bar{g}_\infty \cdot x_i$\;\textrm{+clamp} \\
\bottomrule
\end{tabular}
\par\vspace{1mm}\noindent\footnotesize
$g_\tau = \sigma(\tau \cdot x_i)$, $\bar{g}_\tau = 1 - g_\tau$,
$g_\infty = \mathbf{1}[x_i{>}0]$.
``cal.'' = calibrated per-layer (App.~\ref{app:tau_matching}). For $K{=}2$ rows, $(s_1,s_2)$ and $(c_1,c_2)$ index the $g$-branch first and $\bar g$-branch second; for $K{=}3$ rows, ordered low$\to$high in $x$. With ``$+$clamp'', the lower branch may output $x$ and be saturated by the clamp (e.g., $\FuncReLUSix$ $K{=}2$: $\mathbf{s}{=}(1,1)$, clamp$[0,6]$ gives $\max(0,\min(x,6))$).
\end{table}
\footnotetext{Lemma~\ref{lem:local_residual} matching gives $\tau = 2 f''(0)/(s_+{-}s_-) \approx 1.596$; we report the conventional Hendrycks--Gimpel value $\tau = 1.702$, the global $L^2$-best fit of $\sigma(\tau x)$ to $\Phi(x)$. Per-layer calibration (App.~\ref{app:tau_matching}) absorbs the residual.}

\subsection{Interpretability and Families of Activations}\label{ssec:families}
Reading activations as gates rather than as opaque post-MAC nonlinearities also makes interpretability easier: each branch has a defined role, and what an activation \emph{does} can be read off from $(s_k, c_k)$ directly. As a first instance, standard activations group into a small number of structural families based on what their branches compute (Table~\ref{tab:families}).

\begin{wrapfigure}[7]{r}{0.50\textwidth}
\vspace{-0.5cm}
\begin{minipage}{0.50\textwidth}
\captionof{table}{Structural families from $(s_k, c_k)$.}
\label{tab:families}
\centering\scriptsize
\setlength{\tabcolsep}{2.5pt}
\renewcommand{\arraystretch}{1.0}
\begin{tabular}{@{}lll@{}}
\toprule
\textbf{Family} & \textbf{Condition} & \textbf{Members} \\
\midrule
Mult.\ (exact) & $\exists\,k{:}\,s_k{\neq}0$ & $\FuncReLU,\FuncLeakyReLU,\FuncSiLU$ \\
Mult.\ (approx.) & & $\FuncGELU,\FuncELU,\FuncMish$ \\
Mult.\ (sat.) & & $\FuncHardtanh,\FuncReLUSix$ \\
Constant & $s_k{=}0\;\forall k$ & $\FuncSigmoid,\FuncTanh$ \\
\bottomrule
\end{tabular}
\end{minipage}
\vspace{-0.4cm}
\end{wrapfigure}

Some activations \emph{carry the input forward} through the branches (multiplicative, $s_k \neq 0$), others \emph{use the input only as a gating signal}. We call the former the \emph{multiplicative family}. Within this family, conversion is either \emph{exact} ($\FuncReLU$, $\FuncSiLU$, with intrinsic $\FuncSigmoid$ gates) or \emph{approximate} ($\FuncGELU$, $\FuncELU$, $\FuncMish$~\citep{misra2020mish}); for the approximate ones, calibration reduces the residual to within $0.04\%$ on $\FuncGELU$. The ones where the input is only used for the gating are the \emph{constant family} ($\FuncSigmoid$ and $\FuncTanh$).
\begin{remark}
When $s_k \neq 0$, input magnitude flows through the branch, giving a direct gradient path from $a_j$ to $x_i$; when $s_k = 0$ for all branches, gradients pass through the saturating $g(1{-}g)$, which is why $\FuncSigmoid$/$\FuncTanh$ vanish in feedforward stacks yet work as state interpolators in LSTM/GRU gates.
\end{remark}

\subsection{The Minimal Branch Theorem}
\label{ssec:minimal_branch}

Two questions arise about the scope of the $\FuncTG$
primitive. First, Section~\ref{ssec:derivations}
established equivalence for commonly used activations
(Table~\ref{tab:families}), but what is the full class
of activations that $\FuncTG$ can model? Second, any
continuous function can be approximated to precision
$\varepsilon$ by a piecewise-linear interpolant with
$K = O(1/\varepsilon)$ pieces---and since $\FuncTG$ is
such an interpolant (with $\FuncSigmoid$ gates), it inherits
this guarantee. However, increasing $K$ beyond a small
constant becomes infeasible implementation-wise.

We show that activations suited for training deep neural
networks are structurally constrained in a way that makes
them modelable by $\FuncTG$ with $K \in \{2, 3\}$.
We observe empirically that all activations deployed in
production architectures (Table~\ref{tab:families})
share a common structural property: $f'$ is bounded and
has at most one local extremum. Composite activations ($\FuncMish$, $\FuncGELU$, $\FuncELU$)
are borderline cases, handled in
Remarks~\ref{rem:composition} and~\ref{rem:product}.
We call functions satisfying the single-extremum property
\emph{well-behaved}.

\begin{definition}[Well-behaved activation]
\label{def:wellbehaved}
$f : \Reals \to \Reals$ is \emph{well-behaved} if it is
continuous, piecewise $C^2$ with bounded derivative on
each smooth piece, and has a dominant transition point at
$\theta_f$ (an inflection where $f''$ changes sign, or a
corner where left and right derivatives differ); secondary
sign changes of $f''$ in the tails---where $|f''|$ is
already small---are permitted, and compositions or products
of well-behaved factors are covered in
Remarks~\ref{rem:composition}, \ref{rem:product}.
Boundedness ensures the asymptotic slopes
$s_\pm := \lim_{x \to \pm\infty} f'(x)$ exist, which
determines the affine branches in
Theorem~\ref{thm:minimal_branch}.
\end{definition}

\begin{theorem}[Minimal Branch Theorem]
\label{thm:minimal_branch}
Let $f$ be well-behaved
(Definition~\ref{def:wellbehaved}) with transition point
$\theta_f$ and asymptotic slopes $s_\pm$. There exist
parameters $(\tau, \theta, s_k, c_k)$ such that $K{=}2$
$\FuncTG$ matches the asymptotic slopes of $f$ at
$\pm\infty$ ($\FuncTG(x) \to s_\pm x + \mathrm{const}$)
and has a single transition at $\theta_f$. $K{=}1$ is
insufficient whenever $f$ has a strict transition at
$\theta_f$ (an inflection with $f''$ changing sign, or
a corner).
\end{theorem}
\vspace{-0.6em}
\begin{proof}
Set $L_1(x) = s_+ x + b_+$ and $L_2(x) = s_- x + b_-$
with slopes equal to the asymptotic slopes $s_\pm$ of $f$,
and intercepts $b_\pm$ chosen so that
$L_1(\theta_f) = L_2(\theta_f) = f(\theta_f)$. Take
$g(x) = \sigma(\tau(x - \theta_f))$ and
$\FuncTG(x) = g(x)\,L_1(x) + (1-g(x))\,L_2(x)$. By
construction, $\FuncTG(x) \to s_\pm x + b_\pm$ as
$x \to \pm\infty$ matching the asymptotic slopes of $f$,
and $\FuncTG(\theta_f) = f(\theta_f)$. The transition at
$\theta_f$ is strict whenever $s_+ \neq s_-$. $K{=}1$
yields a single affine function with no transition, so it
cannot match a strict transition of $f$ at $\theta_f$.
\vspace{-0.6em}
\end{proof}

For piecewise-affine cases ($\FuncReLU$,
$\FuncLeakyReLU$), the $K{=}2$ conversion is exact with
residual zero on each affine piece. Saturating activations
with two transitions ($\FuncHardtanh$, $\FuncHardSig$,
$\FuncReLUSix$) require $K{=}3$ directly, or $K{=}2$ when
the saturation branches are absorbed by hardware clamping
(Section~\ref{ssec:derivations}). Smooth activations whose
gate is itself a $\FuncSigmoid$ (e.g., $\FuncSiLU$ or $\FuncTanh$) are also exactly converted via the closed-form
parameters of Section~\ref{ssec:derivations}. The remaining
case---smooth activations whose gate is not a $\FuncSigmoid$
($\FuncGELU$'s Gaussian CDF, $\FuncMish$'s
$\tanh\circ\mathrm{softplus}$)---admits only approximate
$K{=}2$ conversion. A single local scalar quantifies the
residual error.

\subsubsection{Activation Complexity and Trainability}
\label{sssec:activation_complexity}

\vspace{-0.5em}
\begin{lemma}[Local residual]
\label{lem:local_residual}
Let $f$ be well-behaved with $f \in C^3$ near its matching center $\theta_f$ (an inflection of $f$, a corner of $f$, or---in the multiplicative family---the gate center with $f''(\theta_f) \neq 0$).
Construct the $K{=}2$ $\FuncTG$ of
Theorem~\ref{thm:minimal_branch} with $\tau$ chosen so
that $\FuncTG$ agrees with $f$ through the lowest
non-trivial derivative at $\theta_f$: in the multiplicative
family with $f''(\theta_f) \neq 0$, set
$\tau = 2f''(\theta_f)/(s_+ - s_-)$; in the constant family
with $\theta_f$ an inflection of $f$, set
$\tau = 4f'(\theta_f)/(c_1 - c_2)$. Then
$|f(x) - \FuncTG(x)|
   = \tfrac{1}{6}\,\kappa(f)\,|x-\theta_f|^3
   + O\!\bigl((x-\theta_f)^4\bigr)$
near $\theta_f$, where
$\kappa(f) := |f'''(\theta_f)|$ is the
\emph{activation complexity}.
\end{lemma}
\vspace{-0.8em}
\begin{proof}
The matching condition forces $\FuncTG$ and $f$ to agree
at $\theta_f$ through second order: value by intercept
choice, first derivative by $\tau$-matching in the constant
family (trivially in the multiplicative family), and
second derivative by $\tau$-matching in the multiplicative
family (trivially in the constant family, where
$f''(\theta_f) = 0$). Expanding $f$ around $\theta_f$
gives
$f(x) - \FuncTG(x) = \tfrac{1}{6}f'''(\theta_f)(x-\theta_f)^3 + O((x-\theta_f)^4)$.
\vspace{-0.8em}
\end{proof}

The same $\kappa(f)$ governs training stability. For
squared-loss training, a standard analysis around a
well-fit point (Appendix~\ref{app:hessian}) gives:

\begin{corollary}[Stability]
\label{cor:stability}
The maximum stable second-order step size near $\theta_f$
satisfies $\eta = O(1/\sqrt{\kappa(f)})$.
\end{corollary}

\citet{hayou2019impact} bound the maximum trainable depth $L_{\max}$ in random feedforward networks by $\Exp[\phi'(z)^{\,2}] / \Exp[\phi''(z)^{\,2}]$ over Gaussian preactivations ($\phi \equiv f$).
The denominator concentrates near $\theta_f$ and reduces to $\kappa(f)^2$ (Appendix~\ref{app:curvature_concentration}):

\begin{corollary}[Trainable depth]
\label{cor:depth}
In the random-network setting
of~\citet{hayou2019impact}, $L_{\max} \propto 1/\kappa(f)^2$
in the edge-of-chaos regime.
\end{corollary}

\paragraph{Why $K{=}2$ suffices.}
Setting $K{=}2$ enforces well-behavedness
architecturally: a single $\FuncSigmoid$ gate admits at most one
inflection regardless of the learned
$(\tau, \theta, s_k, c_k)$, so the network cannot find
parameters violating the conditions of
Lemma~\ref{lem:local_residual} and its corollaries. Conversely, 
activations with large $\kappa$ are
simultaneously hard to represent and hard to train
deeply, so anything $K{=}2$ misses is uninteresting for
general use. 

\begin{remark}[Compositions]
\label{rem:composition}
$\FuncMish$, $\FuncGELU$, and $\FuncELU$ are compositions
of well-behaved monotone factors. Each component
contributes at most one inflection by monotonicity, giving
$m \leq 2$ and $K \leq 3$ in each case.
\end{remark}

\begin{remark}[Products with identity factor]
\label{rem:product}
For $f(x) = x \cdot h(x)$ with $h$ well-behaved
($\FuncGELU$, $\FuncMish$, $\FuncSiLU$), the identity
factor $p(x) = x$ has $p'' \equiv 0$, collapsing the
worst-case product bound to $m \leq 2$ and giving
$K \leq 3$.
\end{remark}

\subsection{Conversion of Pretrained Networks}
\label{ssec:conversion}

We implement the TG primitive as \textbf{\acp{TGL}}---drop-in
replacements for activation modules. Conversion does not modify the layer type or its
mechanism.
The $\FuncTGL$ module operates per input element \emph{before} the existing operation, and
the layer processes the gated inputs through its original weights.
Conversion is a local operation: for each activation module
$\sigma^{(\ell)}$ in the network, look up
$(K, \theta, \tau, \{s_k\}, \{c_k\})$ from Table~\ref{tab:taxonomy} and
replace $\sigma^{(\ell)}$ with the corresponding $\FuncTGL$ module.
\emph{No layer weights are modified}, no layers are restructured,
and no retraining is required.

\begin{wrapfigure}[12]{r}{0.52\textwidth}
\vspace{-0.5cm}
\begin{minipage}{0.52\textwidth}
\captionof{table}{$\FuncTGL$ conversion without fine-tuning. Vision: ImageNet 50K val unless marked $\dagger$ (5K subset, own baseline). \texttt{timm} default weights. $\Delta$ is intra-row.}
\label{tab:conversion}
\centering\scriptsize
\setlength{\tabcolsep}{2.5pt}
\renewcommand{\arraystretch}{1.1}
\vspace{-0.2cm}
\begin{tabular}{@{}llcrcc@{}}
\toprule
\textbf{Model} & \textbf{Activations} & $K$
  & \textbf{Baseline} & \textbf{$\FuncTGL$} & $\Delta$ \\
\midrule
ResNet-50       & $\FuncReLU$       & 2  & 80.35 & 80.35 & $\mathbf{0.00}$ \\
EffNet-B0       & $\FuncSiLU$+Sig.  & 2  & 77.67 & 77.67 & $\mathbf{0.00}$ \\
\midrule
ViT-B/16        & $\FuncGELU$       & 2  & 81.10 & 80.01 & ${-}1.09$ \\
ViT-B/16        & $\FuncGELU$ cal.  & 2  & 80.86 & 80.46 & ${-}0.40$ \\
ViT-B/16$^\dagger$ & $\FuncGELU$ $K{=}3$ & 3 & 85.14 & 85.14 & $\mathbf{0.00}$ \\
\midrule
LSTM (2-layer)  & $\FuncSigmoid$+$\FuncTanh$  & 2  & 98.09 & 98.09 & $\mathbf{0.00}$ \\
GRU (2-layer)   & $\FuncSigmoid$+$\FuncTanh$  & 2  & 98.89 & 98.89 & $\mathbf{0.00}$ \\
\bottomrule
\end{tabular}
\par\noindent\scriptsize
Vision rows: ImageNet Top-1 (\%); models from \texttt{timm}. All ViT-B/16 rows use the same checkpoint. The K=2 calibrated row evaluates a 49K split disjoint from the 640-image calibration set. The K=3 row ($\dagger$) uses a non-uniform 5K subset of the val set, hence the higher baseline. 
\end{minipage}
\vspace{-0.3cm}
\end{wrapfigure}

We select model--dataset pairs to cover various types of activations in popular architectures: ResNet-50
($\FuncReLU$), EfficientNet-B0 ($\FuncSiLU$ + $\FuncSigmoid$), and ViT-B/16 ($\FuncGELU$) on
ImageNet~\citep{russakovsky2015imagenet}; LSTM/GRU ($\FuncSigmoid$ + $\FuncTanh$) on
MNIST~\citep{lecun1998mnist}. Exact activations ($\FuncReLU$, $\FuncSiLU$, $\FuncSigmoid$, $\FuncTanh$) convert at $\Delta = 0.00\%$ accuracy. GELU requires calibration: the generic $\tau{=}1.702$ fit of $\sigma(\tau x)$ to $\Phi(x)$ has ${\leq}0.3\%$ element-wise error, which per-layer $\tau$ fitting reduces further. ViT-B/16 accuracy improves from $-1.09\%$ (generic) to $-0.40\%$ (per-layer $\tau$) to $0.00\%$ ($K{=}3$).

\section{Training and Expressiveness of TGLs}
\label{sec:learning}\label{ssec:regimes}
\vspace{-0.4em}

\noindent \acp{TGL} are differentiable in soft-gate mode---$g(x;\tau,\theta) = \sigma(\tau(x{-}\theta))$ gives native gradients in all parameters---and remain trainable in hard-gate mode ($\tau \to \infty$) via a subgradient-friendly form of the resulting step function. Training proceeds from scratch or after converting a pretrained model. Beyond drop-in equivalence, the increased expressiveness yields gains along four axes---width reduction, depth reduction, performance, and training speedup---demonstrated below on one representative architecture--dataset pair each. Recall the three tiers of the $\FuncTG$ primitive from Table~\ref{tab:design-space}. The datasets and models are chosen to represent diverse usecases.
\begin{wrapfigure}[26]{r}[-4.2em]{0.55\textwidth}
  \vspace{-1.0em}
  \input{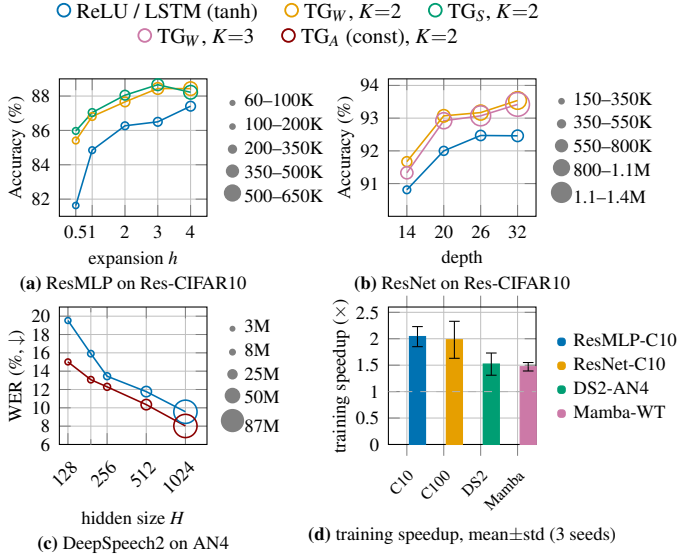}
  \hspace*{0.0em}\begin{minipage}{\dimexpr\linewidth+4.2em\relax}
  \vspace{-0.07cm}
  \caption{Complexity--expressiveness tradeoff.
  (a) ResMLP width on CIFAR-10;
  (b) ResNet depth on CIFAR-10;
  (c) DeepSpeech2-LSTM on AN4, WER vs $H$ ($\downarrow$);
  (d) $\FuncTG_W$ $K{=}2$ training-step speedup over the standard-activation baseline (mean$\pm$std, 3 seeds): CIFAR-10 (C10) vs.\ SiLU, CIFAR-100 (C100) vs.\ GELU, DS2 vs.\ std LSTM, Mamba (MambaMIM/MSD Hipp.) vs.\ SiLU.
  Circle size in (a, b) encodes parameter count.}
  \label{fig:halftransfer}
  \end{minipage}
\end{wrapfigure}

\emph{Width and depth reduction} (Fig.~\ref{fig:halftransfer}(a,b), $\FuncTG_W$): $\FuncTG_W$ inflates parameter count by a factor of $K$, but the increased expressiveness lets the model match or beat its baseline at reduced width or depth. This favors edge deployment, where smaller activation tensors and fewer sequential layers shrink data movement, the dominant cost. On CIFAR-10, a ResMLP with $\FuncTG_W$ $K{=}2$ at expansion $h{=}1$ (Fig.~\ref{fig:halftransfer}(a)) matches a standard ReLU network at $h{=}2$--$3$ at comparable parameter count---the $K{\times}$ branch inflation offset by expressiveness gains permitting reduced expansion; the $\FuncTG_S$ per-synapse variant provides diminished gains compared with $\FuncTG_W$ due to its higher parameter count. Likewise, depth of a DNN can be reduced: Fig.~\ref{fig:halftransfer}(b) shows $K{=}2$ giving the cleanest accuracy/parameter trade-off, with $K{=}3$ adding only marginal gains.

\emph{Performance} (Fig.~\ref{fig:halftransfer}(c), $\FuncTG_A$): given a fixed parameter budget, $\FuncTG_A$ provides better model performance. On the AN4 speech recognition benchmark, placing per-unit $(\tau,\theta)$ on every LSTM gate of a DeepSpeech2 (5-layer bi-LSTM, CTC) improves WER by $1.2$--$4.5$ points across $H{\in}\{128,\dots,1024\}$ at $<\!0.5\%$ parameter overhead. The $\FuncTG_A$ parameterization is constrained to find parameterizations of  \emph{constant family} of the underlying LSTM gates ($s_k{=}0$, fixed $c_k$); releasing this constraint to multiplicative degrades performance back toward the baseline (Appendix~\ref{app:training_details}), verifying the family taxonomy of Section~\ref{ssec:families} and suitability of the \emph{constant family} for RNN gates.

\emph{Training speedup} (Fig.~\ref{fig:halftransfer}(d), $\FuncTG_W$ with per-channel $(\tau,\theta)$): the optimizer reaches the baseline's converged metric in fewer steps as activations adapt across layers. Speedup is $T_{\mathrm{base}}/T_{\mathrm{TGL}}$, where $T$ is the first step at which the rolling-3-smoothed metric crosses the baseline's converged level (mean over its last 6 evals; $\geq$ for accuracy/Dice, $\leq$ for WER). Across four model--dataset pairs (3 seeds each), $\FuncTG_W$ $K{=}2$ yields $1.47$--$2.04\times$ over the corresponding baseline activation. Full setup in Appendix~\ref{app:training_details}.

\section{Softmax Attention as a TG Instance}
\label{sec:attention}

Softmax output at position~$i$ depends on all inputs
through a cross-element reduction. We write $z_i$ for the
pre-softmax logit at position~$i$ (the role $x_i$ plays
in Section~\ref{sec:tgl}) and $z_{-i}$ for the remaining
logits. The well-known identity~\citep{bridle1990softmax}
$\mathrm{softmax}(z)_i = \sigma\!\bigl(z_i -
\mathrm{LSE}(z_{-i})\bigr)$ rewrites cleanly as the
TG primitive of Eq.~\eqref{eq:tg_primitive} with
$K{=}2$ constant branches, $\tau{=}1$, and
data-dependent threshold $\theta_i =
\mathrm{LSE}(z_{-i})$:
\begin{equation}
\label{eq:bridle_tgl}
\mathrm{softmax}(z)_i
\;=\;
\sum_{k=1}^{2} g_k(z_i;\,\tau{=}1,\,\theta_i)\,
\bigl(s_k z_i + c_k\bigr),
\qquad
(s_1, c_1)=(0, 1),\;\;
(s_2, c_2)=(0, 0),\;\;
\theta_i = \mathrm{LSE}(z_{-i}).
\end{equation}
The threshold is computed by cross-element reduction;
the per-position output is then a $\FuncSigmoid$ gated against
that threshold. We call this the
\emph{TG-softmax conversion}: \emph{train with softmax,
convert to the TG form for inference} via
Eq.~\eqref{eq:bridle_tgl}, preserving pretrained weights.
Additionally, $\theta_i = \mathrm{LSE}(z_{-i})$ itself
decomposes recursively via the pairwise identity
$\mathrm{LSE}(a, b) = \max(a, b) +
\mathrm{softplus}(-|a - b|)$: $\max$ is an exact $K{=}2$ TG
instance, and
$\mathrm{softplus}(x) = \int_{-\infty}^{x} g(t;\,1,\,0)\,dt$
is the integral of the TG gate. Together, TG plus one
integration step (a single-stage analog operation)
covers all forward-attention nonlinearity; we leave the
algorithmic and hardware implications to future work.

\begin{wrapfigure}{r}{0.60\textwidth}
\vspace{-1.0em}
\centering
\scriptsize
\captionof{table}{End-to-end TG conversion PPL on
WikiText-2 and OpenWebText ($S{=}1024$).}
\label{tab:e2e_llm_conversion}
\setlength{\tabcolsep}{3pt}
\renewcommand{\arraystretch}{1.05}
\begin{tabular}{@{}l c c@{\hskip 8pt} cc @{\hskip 8pt} cc @{\hskip 6pt} c@{}}
\toprule
& \multicolumn{2}{c}{Activations} & \multicolumn{2}{c}{WikiText-2 PPL} & \multicolumn{2}{c}{OpenWebText PPL} & \\
\cmidrule(lr){2-3}\cmidrule(lr){4-5}\cmidrule(lr){6-7}
Model & FFN & Attn & sm & TG & sm & TG & $\Delta_{\mathrm{wt2}}\,\%$ \\
\midrule
SmolLM-135M    & SiLU & sm$\to$sig & 20.295 & 20.295 & 21.715 & 21.715 & $\mathbf{0.00}$ \\
TinyLlama-1.1B & SiLU & sm$\to$sig & 10.160 & 10.160 & \phantom{0}9.343 & \phantom{0}9.343 & $\mathbf{0.00}$ \\
GPT-2 124M     & GELU & sm$\to$sig & 29.249 & 29.314 & 21.037 & 21.215 & $+0.22$ \\
\bottomrule
\end{tabular}
\vspace{-0.6em}
\end{wrapfigure}

\paragraph{End-to-end conversion.}
Table~\ref{tab:e2e_llm_conversion} reports WikiText-2
and OpenWebText PPL on three pretrained LLMs after
swapping every FFN and attention activation to its TG
instance as a pure architectural rewrite (no fine-tuning
or calibration): non-softmax activations (SiLU, GELU)
via the closed-form parameters of
Section~\ref{ssec:derivations} (Table~\ref{tab:taxonomy}),
softmax via Eq.~\eqref{eq:bridle_tgl}.
SmolLM-135M and TinyLlama-1.1B preserve PPL to
floating-point precision ($\Delta = 0.00$\%): the
conversion is an algebraic identity wherever the
activation maps exactly into a TG primitive. GPT-2's
single non-exact activation (GELU) accounts for the only
drift ($+0.22$\% on WikiText-2)---the same $K{=}2$ GELU
residual seen on ViT-B/16 (Table~\ref{tab:conversion}),
fully recovered ($\Delta = 0.00\%$) at $K{=}3$.

\begin{wrapfigure}[8]{r}{0.50\textwidth}
    \vspace{-2.2em}
    \centering

\begin{minipage}{\linewidth}
\centering
\begin{minipage}[b]{0.45\textwidth}
\centering
\begin{circuitikz}[scale=0.48, transform shape, every node/.style={font=\large}]
  \draw (0,-2.3) node[ground]{} to[isource, l_=\raisebox{-2pt}{$I_b$}] (0,-1.0);
  \draw (0,-1.0) -- (-1.2,-1.0);
  \draw (-1.2,-1.0) node[nmos, anchor=source, xscale=-1](ML){};
  \draw (0,-1.0) -- (1.2,-1.0);
  \draw (1.2,-1.0) node[nmos, anchor=source](MR){};
  \draw (ML.drain) to[R, l_=\raisebox{-2pt}{$R_L$}] ++(0,1.4) coordinate(vddL);
  \draw (vddL) -- ++(0,0.1) node[vcc]{$V_\mathrm{DD}$};
  \draw (MR.drain) to[R, l=\raisebox{-2pt}{$R_L$}] ++(0,1.4) coordinate(vddR);
  \draw (vddR) -- ++(0,0.1) node[vcc]{$V_\mathrm{DD}$};
  \draw (ML.gate) -- ++(-0.4,0) node[left]{$V^+$};
  \draw (MR.gate) -- ++(0.4,0) node[right]{$V^-$};
  \draw (MR.drain) -- ++(0.5,0) node[right]{$V_o$};
\end{circuitikz}
\end{minipage}%
\hfill
\begin{minipage}[b]{0.50\textwidth}
\centering
\begin{circuitikz}[scale=0.48, transform shape, every node/.style={font=\large}]
  \foreach \r in {0,1,2} {
    \draw[thick] (-0.6, -\r*1.6) -- (3.6, -\r*1.6);
    \node[left] at (-0.6, -\r*1.6) {$V_\r$};
  }
  \foreach \c in {0,1,2} {
    \draw[thick] (\c*1.5+0.4, 0.8) -- (\c*1.5+0.4, 0.25);
    \draw (\c*1.5+0.4, 0.25) to[memristor] (\c*1.5+0.4, -0.25);
    \draw[thick] (\c*1.5+0.4, -0.25) -- (\c*1.5+0.4, -1.35);
    \draw (\c*1.5+0.4, -1.35) to[memristor] (\c*1.5+0.4, -1.85);
    \draw[thick] (\c*1.5+0.4, -1.85) -- (\c*1.5+0.4, -2.95);
    \draw (\c*1.5+0.4, -2.95) to[memristor] (\c*1.5+0.4, -3.45);
    \draw[thick] (\c*1.5+0.4, -3.45) -- (\c*1.5+0.4, -3.8);
  }
  \node[font=\large] at (0.96, 0.35) {$G_{00}$};
  \node[font=\large] at (2.46, -1.25) {$G_{11}$};
  \node[font=\large] at (3.96, -2.85) {$G_{22}$};
  \foreach \c in {0,1,2} {
    \node[below] at (\c*1.5+0.4, -3.8) {$I_\c$};
  }
\end{circuitikz}
\end{minipage}
\end{minipage}
    \vspace{-0.2cm}
    \caption{(a) diff-pair $\FuncSigmoid$; (b) RRAM MAC.}
    \label{fig:analog-blocks}
    \vspace{-1.0em}
  \end{wrapfigure}
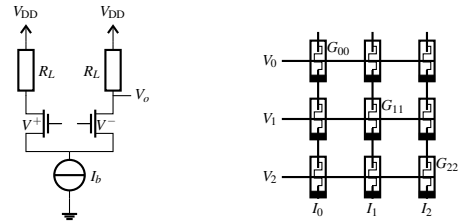

\section{The Case for Analog In-Memory Computing}\label{sec:analog}
Each operation in the TG primitive has a direct analog counterpart.
The tunable $\FuncSigmoid$ $\sigma_{\Sharpness}(x-\GateThresh)$ is realized by a MOSFET
differential pair (Fig.~\ref{fig:analog-blocks}a) whose transfer function
$V_o \propto \tanh\!\bigl(\tfrac{g_m}{2I_b}(V^+-V^-)\bigr)$ produces a $\FuncSigmoid$
with steepness set by $g_m/I_b$, directly implementing~$\Sharpness$, while
the input offset programs~$\GateThresh$;
in the high-$g_m/I_b$ regime, the same circuit saturates into a step function $\mathbf{1}[x>\GateThresh]$, recovering hard-gating activations
(ReLU, Hardtanh) without a separate comparator.
The linear MAC backbone is computed by a resistive crossbar (Fig.~\ref{fig:analog-blocks}b),
where input voltages $V_i$ on wordlines produce column currents $I_j = \sum_i V_i G_{ij}$. A configurable $\FuncSigmoid$ circuit colocated with the crossbar realizes any
desired $(\tau, \theta)$ sharing granularity at modest per-device storage
cost, supported by emerging nonvolatile memories such as FeCAP and
multi-level RRAM~\citep{ielmini2018memory,sebastian2020memory}.
Because TG unifies all standard activations under one parameterized $\FuncSigmoid$, a
single analog front-end suffices for every layer, eliminating bespoke
per-activation circuitry; circuit-level integration and silicon validation are
deferred to future work.

\section{Related Work}
\label{sec:related}

Gating has long appeared in the literature, but no prior work has used it as the fundamental primitive for neural nonlinearity. Highway Networks~\citep{srivastava2015highway} apply a learned scalar gate per layer to mix a transformed branch with the input, leaving the activation inside untouched. SwiGLU~\citep{shazeer2020glu} gates a split projection within a layer rather than the activation itself; Mixture of Experts~\citep{shazeer2017moe} routes inputs across $K$ expert sublayers via a learned gate; Maxout~\citep{goodfellow2013maxout} is a Mixture-of-Experts compressed into one unit: a hard $\max$ over $K$ post-MAC projections with no algebraic inverse to a pretrained weight matrix. A separate line redesigns the activation function itself: APL~\citep{agostinelli2015apl} learns a sum of hinge functions, PAU~\citep{molina2020pau} fits Pad\'e approximants, KAN~\citep{liu2025kan} replaces each weight with a learnable B-spline, and Curvature Tuning~\citep{hu2025curvature} interpolates between two fixed activations. All of these design \emph{better activations} within the conventional paradigm; none reframes activations as gating. TG instead reveals that every standard activation is already a gating primitive: a single TG subsumes ReLU, SiLU, GELU, $\FuncSigmoid$, and Tanh, and explains their trainability through gating.

\section{Limitations and Discussion}
\label{sec:limitations}

Our core contribution is rethinking activations as gating and the link to analog in-memory computing. A few secondary points could be mentioned as limitations: \textbf{(i)} TG kernels are not hand-optimized, so wall-clock training may be slightly slower---our speedup claim is in training \emph{steps}, with Triton kernels giving wall-clock times comparable to baselines, likewise width/depth compression can vary based on different models with room for further research (Section~\ref{ssec:regimes}); \textbf{(ii)} on the circuit side, we provide only a reference topology---detailed implementation and crossbar integration remain future work, and ADC/DAC may still be needed for system integration, though TG removes one major obstacle; \textbf{(iii)} specialized activations with many inflection points exist but are confined to narrow use cases---our work targets the activations used in general deep-network training.

\section{Conclusion and Future Directions}
\label{sec:conclusion}
Activation choice has been a discrete architectural decision (ReLU vs.\ GELU vs.\ SiLU), and hardware---analog and digital alike---has accumulated bespoke per-activation paths: separate circuits for $\FuncSigmoid$, $\FuncTanh$, exp, and the Gaussian CDF, separate kernels for each on GPUs, separate approximation tables on FPGAs. Under $\FuncTG$, activation choice becomes a continuous per-neuron parameter selection within a single primitive, the activation itself becomes implicit---fused into branch-weight gating rather than applied as a post-MAC nonlinearity---a single front-end suffices for every layer, and the same unification extends to softmax attention. The shift carries direct energy and area benefits. Beyond hardware, the same primitive carries advantages in trainability (Section~\ref{ssec:minimal_branch}), training efficiency, model compression, and interpretability of learned per-neuron parameters. 

A few directions emerge for future work. \textbf{(i) Analog in-memory SoC:} our work removes one major obstacle to analog in-memory computing by unifying every activation under a single tunable $\FuncSigmoid$ front-end, but a deployable analog substrate must still address system integration, device variation, exact circuit implementation, etc.
\textbf{(ii) Digital HW:} since we show that $K{=}2$ $\FuncTG$ covers every standard activation, hardware can specialize to a single parameterized $\FuncSigmoid$ evaluator instead of supporting exp, $\FuncTanh$, and the Gaussian CDF separately. The same unification argument that motivates the analog substrate of Section~\ref{sec:analog} applies to digital accelerators (FPGA or ASIC-based for example). \textbf{(iii) Interpretability:} The distributions of $\FuncTG$ parameters across training stages and layers can be used for explaining the model and further improving training recipes.

\color{black}
\printbibliography

@inproceedings{bridle1990softmax,
author       = {John S. Bridle},
  editor       = {David S. Touretzky},
  title        = {Training Stochastic Model Recognition Algorithms as Networks can Lead to Maximum Mutual Information Estimation of Parameters},
  booktitle    = {Proceedings of the Conference on Advances in Neural Information Processing Systems (NIPS)},
      venue={Denver, CO, USA},
	date={1989-11-27/1989-11-30},
  pages        = {211--217},
 url = {https://proceedings.neurips.cc/paper/1989/hash/0336dcbab05b9d5ad24f4333c7658a0e-Abstract.html},
 volume = {2},
  publisher    = {Morgan Kaufmann}
}

@inproceedings{nair2010relu,
author       = {Vinod Nair and
                  Geoffrey E. Hinton},
  editor       = {Johannes F{\"{u}}rnkranz and Thorsten Joachims},
  title        = {Rectified Linear Units Improve Restricted Boltzmann Machines},
  booktitle    = {Proceedings of the 27th International Conference on Machine Learning (ICML)},
  venue={Haifa, Israel},
  date={2010-06-21/2010-06-24},
  pages        = {807--814},
  publisher    = {Omnipress},
  url          = {https://icml.cc/Conferences/2010/papers/432.pdf},
}

@article{prezioso2015,
title={Training and Operation of an Integrated Neuromorphic Network Based on Metal-Oxide Memristors},
author={Mirko Prezioso and
                  Farnood Merrikh{-}Bayat and
                  Brian Hoskins and
                  Gina C. Adam and
                  Konstantin K. Likharev and
                  Dmitri B. Strukov},
journal={Nature},
volume={521},
number={7550},
pages={61--64},
year={2015},
doi={10.1038/nature14441}
}

@inproceedings{hayou2019impact,
  author       = {Soufiane Hayou and
                    Arnaud Doucet and
                    Judith Rousseau},
    editor       = {Kamalika Chaudhuri and
                    Ruslan Salakhutdinov},
    title        = {On the Impact of the Activation function on Deep Neural Networks Training},
    booktitle    = {Proceedings of the 36th International Conference on Machine Learning (ICML)},
	venue={Long Beach, CA, USA},
	date={2019-06-09/2019-06-15},
    pages        = {2672--2680},
    publisher    = {PMLR},
    url          = {http://proceedings.mlr.press/v97/hayou19a.html}
}

@article{hu2025curvature,
  title   = {Curvature Tuning: {Provable} Training-free Model Steering From a Single Parameter},
  author  = {Hu, Leyang and Gamba, Matteo and Balestriero, Randall},
  year    = {2025},
  journal      = {The Computing Research Repository (CoRR)},
  eprint={2502.07783},
  eprintclass={cs.LG},
  eprinttype   = {arxiv}
}

@inproceedings{liu2025kan,
author       = {Ziming Liu and
                  Yixuan Wang and
                  Sachin Vaidya and
                  Fabian Ruehle and
                  James Halverson and
                  Marin Soljacic and
                  Thomas Y. Hou and
                  Max Tegmark},
  title        = {{KAN:} {Kolmogorov-Arnold} Networks},
  booktitle    = {Proceedings of the Thirteenth International Conference on Learning Representations (ICLR)},
  venue={Singapore},
  date={2025-04-24/2025-04-28},
  publisher    = {OpenReview.net},
  url          = {https://openreview.net/forum?id=Ozo7qJ5vZi},
}

@inproceedings{he2015prelu,
author       = {Kaiming He and
                  Xiangyu Zhang and
                  Shaoqing Ren and
                  Jian Sun},
  title        = {Delving Deep into Rectifiers: {Surpassing} Human-Level Performance on
                  {ImageNet} Classification},
  booktitle    = {Proceedings of the IEEE International Conference on Computer Vision (ICCV)},
  venue = {Santiago, Chile}, 
  date={2015-12-07/2015-12-13},
  pages        = {1026--1034},
  publisher    = {IEEE},
  doi          = {10.1109/ICCV.2015.123},
}

@inproceedings{maas2013leakyrelu,
  author    = {Andrew L. Maas and Awni Y. Hannun and Andrew Y. Ng},
  title     = {Rectifier Nonlinearities Improve Neural Network Acoustic Models},
  booktitle = {Proceedings of the ICML Workshop on Deep Learning for Audio, Speech and Language Processing (WDLASL)},
  venue={Atlanta, GA, USA},
  pages={1--6},
  date      = {2013-06-16},
}

@inproceedings{clevert2016elu,
author       = {Djork{-}Arn{\'{e}} Clevert and
                  Thomas Unterthiner and
                  Sepp Hochreiter},
  editor       = {Yoshua Bengio and
                  Yann LeCun},
  title        = {Fast and Accurate Deep Network Learning by Exponential Linear Units ({ELUs})},
  booktitle    = {Proceedings 4th International Conference on Learning Representations (ICLR)},
  venue={San Juan, Puerto Rico},
  date={2016-05-02/2016-05-04},
    eprint={1511.07289},
    eprintclass={cs.LG},
    eprinttype   = {arxiv}
}

@inproceedings{ramachandran2018swish,
  author       = {Prajit Ramachandran and
                    Barret Zoph and
                    Quoc V. Le},
    title        = {Searching for Activation Functions},
    booktitle    = {Workshop Track Proceedings of the 6th International Conference on Learning Representations (ICLR)},
	venue={Vancouver, BC, Canada},
	date={2018-04-30/2018-05-03},
    publisher    = {OpenReview.net},
    year         = {2018},
    url          = {https://openreview.net/forum?id=Hkuq2EkPf}
}

@article{hendrycks2016gelu,
  author    = {Dan Hendrycks and Kevin Gimpel},
  title     = {Gaussian Error Linear Units ({GELUs})},
  journal      = {The Computing Research Repository (CoRR)},
  eprint={1606.08415},
  eprintclass={cs.LG},
  eprinttype   = {arxiv},
  year      = {2016}
}

@inproceedings{agostinelli2015apl,
  author       = {Forest Agostinelli and
                    Matthew D. Hoffman and
                    Peter J. Sadowski and
                    Pierre Baldi},
  editor       = {Yoshua Bengio and
                    Yann LeCun},
  title        = {Learning Activation Functions to Improve Deep Neural Networks},
  booktitle    = {Workshop Track Proceedings of the 3rd International Conference on Learning Representations (ICLR)},
  venue={San Diego, CA, USA},
  date         = {2015-05-07/2015-05-09},
    eprint={1412.6830},
    eprintclass={cs.NE},
    eprinttype   = {arxiv}
}

@inproceedings{molina2020pau,
  author       = {Alejandro Molina and
                   Patrick Schramowski and
                   Kristian Kersting},
   title        = {Pad{\'{e}} Activation Units: End-to-end Learning of Flexible Activation Functions in Deep Networks},
   booktitle    = {Proceedings of the 8th International Conference on Learning Representations (ICLR)},
   venue = {Addis Ababa, Ethiopia},
   date={2020-04-26/2020-04-30},
   publisher    = {OpenReview.net},
   year         = {2020},
   url          = {https://openreview.net/forum?id=BJlBSkHtDS}
}

@inproceedings{goodfellow2013maxout,
author       = {Ian J. Goodfellow and
                  David Warde{-}Farley and
                  Mehdi Mirza and
                  Aaron C. Courville and
                  Yoshua Bengio},
  title        = {Maxout Networks},
  booktitle    = {Proceedings of the 30th International Conference on Machine Learning (ICML)},
  venue={Atlanta, GA, USA},
  date={2013-06-16/2013-06-21},
  pages        = {1319--1327},
  publisher    = {JMLR.org},
  url          = {http://proceedings.mlr.press/v28/goodfellow13.html}
}

@article{shazeer2020glu,
  author    = {Noam Shazeer},
  title     = {{GLU} Variants Improve Transformer},
  journal      = {The Computing Research Repository (CoRR)},
  eprint={2002.05202},
  year={2020},
  eprintclass={cs.LG},
  eprinttype   = {arxiv}
}

@inproceedings{shazeer2017moe,
  author       = {Noam Shazeer and
                    Azalia Mirhoseini and
                    Krzysztof Maziarz and
                    Andy Davis and
                    Quoc V. Le and
                    Geoffrey E. Hinton and
                    Jeff Dean},
    title        = {Outrageously Large Neural Networks: The Sparsely-Gated Mixture-of-Experts
                    Layer},
    booktitle    = {Proceedings of the 5th International Conference on Learning Representations (ICLR)},
	venue={Toulon, France},
	date={2017-04-24/2017-04-26},
    publisher    = {OpenReview.net},
    year         = {2017},
    url          = {https://openreview.net/forum?id=B1ckMDqlg},
}

@inproceedings{shafiee2016isaac,
  author       = {Ali Shafiee and
                    Anirban Nag and
                    Naveen Muralimanohar and
                    Rajeev Balasubramonian and
                    John Paul Strachan and
                    Miao Hu and
                    R. Stanley Williams and
                    Vivek Srikumar},
    title        = {{ISAAC:} {A} Convolutional Neural Network Accelerator with In-Situ Analog Arithmetic in Crossbars},
    booktitle    = {Proceedings of the 43rd ACM/IEEE Annual International Symposium on Computer Architecture (ISCA)},
	venue={Seoul, South Korea},
	date={2016-06-18/2016-06-22},
    pages        = {14--26},
    publisher    = {IEEE},
    doi          = {10.1109/ISCA.2016.12}
}

@inproceedings{chi2016prime,
  author       = {Ping Chi and
                    Shuangchen Li and
                    Cong Xu and
                    Tao Zhang and
                    Jishen Zhao and
                    Yongpan Liu and
                    Yu Wang and
                    Yuan Xie},
    title        = {{PRIME:} {A}{} Novel Processing-in-Memory Architecture for Neural Network
                    Computation in {ReRAM}-Based Main Memory},
    booktitle    = {Proceedings of the 43rd {ACM/IEEE} Annual International Symposium on Computer Architecture (ISCA)},
	venue={Seoul, South Korea},
	date={2016-06-18/2016-06-22},
    pages        = {27--39},
    publisher    = {IEEE},
    doi          = {10.1109/ISCA.2016.13}
}

@article{russakovsky2015imagenet,
  author    = {Olga Russakovsky and Jia Deng and Hao Su and Jonathan Krause and Sanjeev Satheesh and Sean Ma and Zhiheng Huang and Andrej Karpathy and Aditya Khosla and Michael Bernstein and Alexander C. Berg and Li Fei-Fei},
  title     = {{ImageNet} Large Scale Visual Recognition Challenge},
  journal   = {International Journal of Computer Vision},
  volume    = {115},
  number    = {3},
  pages     = {211--252},
  year      = {2015},
  doi          = {10.1007/S11263-015-0816-Y}
}

@software{lecun1998mnist,
  author    = {Yann LeCun and Corinna Cortes and Christopher J. C. Burges},
  title     = {The {MNIST} Database of Handwritten Digits},
  url   = {http://yann.lecun.com/exdb/mnist/},
  year      = {1998},
}

@article{ielmini2018memory,
  author    = {Daniele Ielmini and H.-S. Philip Wong},
  title     = {In-Memory Computing with Resistive Switching Devices},
  journal   = {Nature Electronics},
  volume    = {1},
  number    = {6},
  pages     = {333--343},
  year      = {2018},
  doi={10.1038/s41928-018-0092-2}
}

@article{sebastian2020memory,
  author    = {Abu Sebastian and Manuel Le Gallo and Riduan Khaddam-Aljameh and Evangelos Eleftheriou},
  title     = {Memory Devices and Applications for In-Memory Computing},
  journal   = {Nature Nanotechnology},
  volume    = {15},
  number    = {7},
  pages     = {529--544},
  year      = {2020},
  doi={10.1038/s41565-020-0655-z}
}

@article{Leroux2025,
  author  = {Leroux, Nathan and Manea, Paul-Philipp and Sudarshan, Chirag
             and Finkbeiner, Jan and Siegel, Sebastian and Strachan, John Paul
             and Neftci, Emre},
  title   = {Analog In-Memory Computing Attention Mechanism for Fast and Energy-Efficient Large Language Models},
  journal = {Nature Computational Science},
  year    = {2025},
  volume  = {5},
  number  = {9},
  pages   = {813--824},
  doi     = {10.1038/s43588-025-00854-1}
}

@inproceedings{srivastava2015highway,
 author       = {Rupesh Kumar Srivastava and
                   Klaus Greff and
                   J{\"{u}}rgen Schmidhuber},
   editor       = {Corinna Cortes and
                   Neil D. Lawrence and
                   Daniel D. Lee and
                   Masashi Sugiyama and
                   Roman Garnett},
   title        = {Training Very Deep Networks},
   booktitle    = {Proceedings of the Annual Conference on Advances in Neural Information Processing Systems (NIPS)},
	 date={2015-12-07/2015-12-12},
         venue={Montreal, Quebec, Canada},
   pages        = {2377--2385},
   url          = {https://proceedings.neurips.cc/paper/2015/hash/215a71a12769b056c3c32e7299f1c5ed-Abstract.html},
}

@inproceedings{misra2020mish,
  author       = {Diganta Misra},
  title        = {Mish: A Self Regularized Non-Monotonic Activation Function},
  booktitle    = {Proceedings of the 31st British Machine Vision Conference (BMVC)},
  publisher    = {BMVA Press},
  year         = {2020},
  eprint       = {1908.08681},
  eprintclass  = {cs.LG},
  eprinttype   = {arxiv}
}

\appendix

\paragraph{Notation.} Throughout the appendix, we write $f$ for a generic
scalar activation. Where we connect to the trainability framework
of~\citet{hayou2019impact}, the same activation is denoted $\phi$ to match
their convention; $f \equiv \phi$ in all such expressions.

\section{Almost Everywhere Equality}
\label{app:measure_zero}

We show that
$\sigma(\tau(x{-}\theta)) \xrightarrow{\tau \to \infty} \mathbf{1}[x>\theta]$
almost everywhere with respect to the Lebesgue measure on $\mathbb{R}$.
Fix $x \in \mathbb{R} \setminus \{\theta\}$. If $x>\theta$, then
$\tau(x-\theta) \to +\infty$, so
$\sigma(\tau(x-\theta)) = 1/(1 + e^{-\tau(x-\theta)}) \to 1
= \mathbf{1}[x>\theta]$. If $x<\theta$, the same expression
$\to 0 = \mathbf{1}[x>\theta]$. Pointwise convergence therefore holds on
$\mathbb{R}\setminus\{\theta\}$, whose complement $\{\theta\}$ has Lebesgue
measure zero. Hence, the two functions agree a.e., their $L^p$ equivalence
classes coincide for $p\in[1,\infty]$, and the measure-zero disagreement at
$x=\theta$ contributes nothing to summations or integrals used elsewhere in
this work.

\section{Curvature Concentration at the Inflection}
\label{app:curvature_concentration}

This appendix proves Corollary~\ref{cor:depth} by deriving
$\Exp[\phi''(z)^{\,2}] \approx \kappa(f)^2 (q + \theta_f^2)$ under
the Gaussian preactivations of Hayou et al.'s edge-of-chaos analysis.
For a well-behaved activation $f$ (Definition~\ref{def:wellbehaved}), $f''$
has its dominant magnitude at the transition point $\theta_f$ and decays in
both tails; any secondary sign changes far from $\theta_f$ contribute
negligible curvature energy under the Gaussian preactivation. Expanding around $\theta_f$ with $f''(\theta_f)=0$,
\[
f''(z) \;=\; f'''(\theta_f)\,(z-\theta_f) \;+\; O\!\bigl((z-\theta_f)^2\bigr),
\]
so $f''(z)^2 \approx \kappa(f)^2 (z-\theta_f)^2$ in a neighborhood of
$\theta_f$, where $\kappa(f) := |f'''(\theta_f)|$. For the pre-activation
distribution $z \sim \mathcal{N}(0, q)$ used by~\citet{hayou2019impact},
\[
\Exp\!\bigl[(z-\theta_f)^2\bigr]
   \;=\; q + \theta_f^2,
\]
giving
\[
\Exp\!\bigl[\phi''(z)^{\,2}\bigr]
   \;\approx\; \kappa(f)^2 \bigl(q + \theta_f^2\bigr).
\]
The approximation is tight when the bulk of the Gaussian mass sits within the
neighborhood of $\theta_f$ where the cubic expansion holds; for the
activations of Table~\ref{tab:taxonomy}, this neighborhood comfortably
contains $[-3\sqrt{q},\,3\sqrt{q}]$ at all depths used in practice. The
result links a purely local property of $f$ ($\kappa$ at the inflection) to a
global trainability quantity ($\Exp[\phi''^{\,2}]$ at depth).
The $(q + \theta_f^2)$ factor is depth-dependent through $q$: in the
edge-of-chaos regime of~\citet{hayou2019impact}, $q$ converges to a
fixed-point variance $q^*$ as depth grows, so the expectation tends to
$\kappa(f)^2(q^*+\theta_f^2)$ deep in the network and the trainability bound
$L_{\max} \propto 1/\kappa(f)^2$ becomes a property of the activation alone.

\section{Hessian Sensitivity at the Inflection}
\label{app:hessian}

This appendix proves Corollary~\ref{cor:stability}. We use the standard
Gauss--Newton approximation for squared-loss training: for
$\mathcal{L} = \tfrac{1}{2}(y - f(z))^2$ with pre-activation
$z = w^\top x$, the Hessian with respect to $w$ is dominated by the
positive-semidefinite term
\[
H \;\approx\; f'(z)^2\, x x^\top + (y - f(z))\,f''(z)\, x x^\top.
\]
Near a well-fit operating point, the residual term $(y-f(z))f''(z)$ is
small, and we drop it (equivalent to using the empirical Fisher
information matrix). At the inflection $z=\theta_f$, $f''(\theta_f)=0$, so
\[
H(\theta_f) \;=\; f'(\theta_f)^2\, x x^\top.
\]
Differentiating along the pre-activation direction,
\[
\frac{dH}{dz}
   \;=\; 2\,f'(z)\,f''(z)\, x x^\top,
\]
which evaluates to $0$ at $z = \theta_f$ since $f''(\theta_f) = 0$.
The Hessian is therefore locally flat to first order. Continuing to
second order,
\[
\frac{d^2 H}{dz^2}
   \;=\; 2\bigl[f''(z)^2 + f'(z)\,f'''(z)\bigr]\, x x^\top,
\]
which at $\theta_f$ reduces to
$2\,f'(\theta_f)\,f'''(\theta_f)\, x x^\top
   \propto \pm\,\kappa(f)$
since $f''(\theta_f) = 0$. Hence, the Hessian curves at second order with
rate proportional to the activation complexity $\kappa(f)$. A second-order
step size $\eta$ calibrated at $z = \theta_f$ remains valid only while
$\|H(z) - H(\theta_f)\|$ is small relative to $H(\theta_f)$. Bounding the
change by a fixed fraction $\delta$,
\[
\|H(\theta_f+\eta) - H(\theta_f)\|
   \;\approx\; |f'(\theta_f)|\,\kappa(f)\,\eta^2\,\|x x^\top\|
   \;\le\; \delta\,f'(\theta_f)^2\,\|x x^\top\|,
\]
which gives $\eta^2 \le \delta\,|f'(\theta_f)| / \kappa(f)$, i.e.,
$\eta = O(1/\sqrt{\kappa})$. Large $\kappa$ therefore forces small step
sizes near the inflection, the same quantity that governs both the
$\FuncTG$ approximation residual (Section~\ref{ssec:minimal_branch}) and
the trainable depth bound of~\citet{hayou2019impact}.

\section{Training Details}
\label{app:training_details}

This appendix collects the hyperparameters, dataset splits, and seeds used
for each Section~\ref{ssec:regimes} experiment. All runs use a single
NVIDIA RTX 4000 Ada Generation (20\,GB) and the seed is held fixed within
each experiment.

\paragraph{ResMLP width scan (Fig.~\ref{fig:halftransfer}(a)).}
Half-transfer protocol on CIFAR-10: weights initialised from a ReLU baseline,
$(\tau,\theta)$ trained from scratch. $K{=}2$, 300 epochs, seed~42. Both ReLU
baseline and $\FuncTG_W$ runs use identical optimizer, schedule, and
total epoch budget. The $\FuncTG_{S}$ per-synapse variant uses Triton kernels
for the per-weight fused implementation.

\paragraph{ResNet depth scan (Fig.~\ref{fig:halftransfer}(b)).}
ResNet (CIFAR, 3-stage, BasicBlock) on CIFAR-10. Trained from scratch for
200 epochs, seed~42; ReLU baseline and $\FuncTG_W$ runs share the same
optimizer and total epoch budget. $\FuncTG_W$ convolutions implemented as
$K$-branch fused single-GEMM via grouped convolution; per-layer variants use
one scalar $(\tau,\theta)$ per layer, per-channel variants use one
$(\tau,\theta)$ per input channel. Step-speedup is computed as the ratio of
epochs the baseline takes to first reach its peak test accuracy to the
epochs $\FuncTG_W$ takes to first reach the same target (e.g., $d{=}14$:
baseline reaches peak $90.81\%$ at epoch 199, $\FuncTG_W$ reaches it at
epoch 148; speedup $\approx 1.34\times$). The same protocol applies to
panel (a) and panel (c).

\paragraph{DeepSpeech2 / AN4 (Fig.~\ref{fig:halftransfer}(c)).}
DeepSpeech2 with a 2-layer convolutional spectrogram front-end and a 5-layer
bidirectional LSTM decoded with CTC. AN4 splits: 853 train / 95 val / 130
test (CMU census utterances). AdamW with main-LR $3\!\times\!10^{-4}$,
$\FuncTG_A$-LR $7.5\!\times\!10^{-5}$ (scale 0.25), main weight decay
$10^{-5}$, $\FuncTG_A$ weight decay 0, ExponentialLR ($\gamma{=}0.99$),
batch size 8, gradient clip 400, seed~42. $\FuncTG_A$ parameters per LSTM
unit: 5 gates~$\times$~2 $(\tau,\theta)$ scalars per direction. The $H{=}1024$ point uses
equivalence-init from a pretrained checkpoint (WER 9.573\%).

\emph{Family-preservation ablation.} LSTM gates are constant-family ($s_k{=}0$, $c_k$ fixed at $\{0,1\}$ for sigmoid gates and $\{-1,+1\}$ for tanh gates). Plotted panel (3c) preserves this. Releasing the constraint (allowing multiplicative branches with learnable $s_k$) violates the gate's algebraic structure and degrades WER performance:
\begin{center}\small
\begin{tabular}{@{}lcccc@{}}
\toprule
recipe at LSTM gates & $H{=}128$ & $H{=}192$ & $H{=}256$ & $H{=}512$ \\
\midrule
$\FuncTG_A$ $K{=}2$, constant (preserved) & 15.01 & 13.07 & 12.29 & 10.35 \\
$\FuncTG_A$ $K{=}2$, multiplicative (violated) & 16.82 & 15.52 & 13.58 & 12.42 \\
std LSTM (constant, frozen) & 19.53 & 15.91 & 13.45 & 11.77 \\
\bottomrule
\end{tabular}
\end{center}
WER (\%, lower better). The constant-family parameterization wins at every $H$; the multiplicative variant is no better than std LSTM at $H{\geq}256$.

\paragraph{Training-step speedup (Fig.~\ref{fig:halftransfer}(d)).}
We measure step-speedup on four dataset--model pairs with 3 seeds each
(CIFAR-10/SiLU, CIFAR-100/GELU, DS2/AN4 standard LSTM, MambaMIM/MSD
Hippocampus SiLU). For each (dataset, seed) the baseline's converged
level is the mean of its primary metric over its last 6 evals; both
trajectories are rolling-3 smoothed; speedup is the ratio of the baseline's
first-smoothed-crossing of that level to the $\FuncTG_W$ recipe's
first-smoothed-crossing ($\geq$ for accuracy/Dice, $\leq$ for WER). Same
metric for both methods on each pair (test accuracy for CIFAR, multi-class
Dice for MambaMIM, WER for DS2).
Per-pair 3-seed (mean$\pm$std): C10 $2.04\pm0.19\times$, C100
$1.98\pm0.35\times$, DS2 $1.52\pm0.21\times$, Mamba $1.47\pm0.08\times$.
MambaMIM specifics: hybrid (Mamba SSM + CNN U-Net, ${\sim}42$\,M params),
image size $64^3$, 3 classes, AdamW main-LR $10^{-4}$, TG activation-LR
$4\!\times\!10^{-4}$ (scale 4), cosine LR with 500-step warmup, batch size 1,
gradient clip 1.0, EMA decay 0.9999, 30k training steps; all 24
SiLU/GELU/LeakyReLU activations replaced with $\FuncTG_W$ $K{=}2$ at
per-channel $(\tau,\theta)$ sharing, $c_2$ initialized uniformly to
$-0.05$; the SiLU baseline shares optimizer and augmentation recipe.
Across step-matched checkpoints (single-seed Mamba run), $\FuncTG_W$ wins
50/57 (sign-test $p\!\approx\!4\!\times\!10^{-9}$).

\section{LLM Conversion: Procedure and Per-Component Breakdown}
\label{app:llm_conversion}

The end-to-end conversion of Section~\ref{sec:attention} replaces every
activation in a pretrained Transformer in two passes: (i) every attention
softmax is rewritten in TG form via Eq.~\eqref{eq:bridle_tgl}, and (ii)
every FFN activation (SiLU / GELU) is replaced by its closed-form TG
instance from Table~\ref{tab:taxonomy}. No fine-tuning, no calibration data,
no optimiser steps. 

The conversion has two algebraically distinct parts that we report
separately. Let $\mathrm{PPL}_{\mathrm{sm}}$ be the unmodified softmax
baseline, $\mathrm{PPL}_{\mathrm{LSE\textrm{-}\theta}}$ the perplexity with
softmax replaced by $\sigma(z_i - \mathrm{LSE}(z_{-i}))$ but the FFN
activation untouched, $\mathrm{PPL}_{\mathrm{TG\,act}}$ the perplexity with
the FFN activation replaced but softmax untouched, and
$\mathrm{PPL}_{\mathrm{both}}$ the full conversion. The four columns
isolate which step contributes to which residual.

\begin{table}[h]
\centering\scriptsize
\caption{Per-component WikiText-2 PPL of TG conversion. The
$\mathrm{LSE}$-$\theta$ column is exact (Eq.~\eqref{eq:bridle_tgl}); any
non-zero residual must come from the FFN activation column.}
\label{tab:llm_conversion_breakdown}
\setlength{\tabcolsep}{4pt}
\begin{tabular}{@{}llcrrrr@{}}
\toprule
Model & FFN act. & $L_{\mathrm{conv}}$ & sm & LSE-$\theta$ & TG-act & both \\
\midrule
SmolLM-135M     & SiLU & 30 & 20.295 & 20.295 & 20.295 & 20.295 \\
TinyLlama-1.1B  & SiLU & 22 & 10.160 & 10.160 & 10.160 & 10.160 \\
GPT-2 124M      & GELU & 12 & 29.249 & 29.249 & 29.314 & 29.314 \\
\bottomrule
\end{tabular}
\end{table}

For SmolLM-135M and TinyLlama-1.1B every column matches the
softmax baseline to four decimal places: SiLU is an exact $K{=}2$ TG
instance and the Bridle identity is algebraic, so both passes are
floating-point identities. For GPT-2, the LSE-$\theta$ column is also
identical to the softmax baseline; the entire $+0.22$\% drift in the
``both'' column is contributed by the TG-activation pass and is therefore
attributable to the $K{=}2$ approximation of GELU at all 12 layers, with no
contribution from the softmax conversion. A post-conversion structural audit confirms all $L_{\mathrm{conv}}$ activation
modules were replaced and that no original GELU/SiLU modules remain in
the converted graph.

\section{\texorpdfstring{$\tau$}{tau} Values in Table~\ref{tab:taxonomy} as Matching Conditions}
\label{app:tau_matching}

The $\tau$ values in Table~\ref{tab:taxonomy} are instances of a single
matching principle: $\tau$ is chosen so that $\FuncTG$ agrees with $f$
through the lowest non-trivial Taylor term at $\theta_f$
(Lemma~\ref{lem:local_residual}). For the multiplicative family
($s_+ \neq s_-$) this gives
$\tau = 2\,f''(\theta_f)/(s_+ - s_-)$; for the constant family
($s_+ = s_- = 0$),
$\tau = 4\,f'(\theta_f)/(c_1 - c_2)$. Direct evaluation:
\begin{itemize}
\setlength{\itemsep}{0pt}\setlength{\parsep}{0pt}\setlength{\topsep}{2pt}
\item $\FuncSiLU$: $f''(0) = \tfrac{1}{2}$, $s_+ - s_- = 1$
$\Rightarrow \tau = 1$.
\item $\FuncSigmoid$: $f'(0) = \tfrac{1}{4}$, $c_1 - c_2 = 1$
$\Rightarrow \tau = 1$.
\item $\FuncTanh$: $f'(0) = 1$, $c_1 - c_2 = 2$
$\Rightarrow \tau = 2$.
\end{itemize}
recovering all three Table~\ref{tab:taxonomy} values exactly. $\FuncGELU$ is
a multiplicative-family activation with $f''(0) = 2\phi(0) \neq 0$, so the
multiplicative branch of Lemma~\ref{lem:local_residual} applies at
$\theta = 0$ with the formula $\tau = 2 f''(0) / (s_+ - s_-)$, yielding
$\tau \approx 1.596$. The widely-used Hendrycks--Gimpel value $\tau = 1.702$
is instead a global $L^2$-best fit of $\sigma(\tau x)$ to
$\Phi(x)$~\citep{hendrycks2016gelu}, which we adopt for compatibility.
Per-layer calibration (below) absorbs the resulting residual.

\paragraph{$K{\geq}3$ gating.} For $K{=}3$ we use product-form unnormalized gates renormalized to a partition of unity:
\[
\tilde g_1(x) = \sigma(\tau(\theta_1 - x)),\quad
\tilde g_2(x) = \sigma(\tau(x - \theta_1))\,\sigma(\tau(\theta_2 - x)),\quad
\tilde g_3(x) = \sigma(\tau(x - \theta_2)),
\]
with $g_k = \tilde g_k / \sum_j \tilde g_j$ and thresholds $\theta_1 < \theta_2$. Each $g_k$ is a soft indicator for region $k$.

\paragraph{$\FuncGELU$ calibration (Table~\ref{tab:taxonomy}, $K{=}3$ row).} The nine per-layer parameters $(\theta_1, \theta_2, \tau, s_1, s_2, s_3, c_1, c_2, c_3)$ are fit by minimizing per-layer RMS error between $\FuncTG$ output and the true $\FuncGELU$ output on collected pre-activations. Pre-activations are obtained by passing 640 ImageNet validation images (10 batches of 64) through the pretrained ViT-B/16, hooking every $\FuncGELU$ site, and subsampling each layer's activations to $10^5$ values. Optimization uses Nelder--Mead (max 3000 iterations, $\mathrm{xatol} = 10^{-6}$) initialized at $(\theta_1, \theta_2, \tau) = (-1, +1, 1.8)$, $s = (0, 1, 1)$, $c = (0, 0, 0)$, with $\theta_1 < \theta_2$ enforced. The test images used for accuracy evaluation (Table~\ref{tab:conversion}, $\FuncGELU$ $K{=}3$ row) are disjoint from the 640-image calibration set.

\newpage

\end{document}